\newif\ifarxiv
\arxivtrue
\newif\ifdraft
\draftfalse

\documentclass[a4paper]{article}
\usepackage[a4paper]{geometry}

\usepackage{wrapfig}

\usepackage{ifthen}

\usepackage[inline]{enumitem}

\usepackage{natbib}
\usepackage{amsmath,amssymb}
\usepackage{amsthm}
\usepackage{bm}
\usepackage[usenames,dvipsnames,svgnames,table]{xcolor}
\usepackage[hyperindex,
            linktocpage=true,
            colorlinks=true,
            linkcolor=blue,
            urlcolor=blue,
            citecolor=blue,
            anchorcolor=blue
            ]{hyperref}
\usepackage{subcaption}
\usepackage{stmaryrd}

\ifdraft
\usepackage[notref,notcite]{showkeys}
\fi

\usepackage{algorithm}
\usepackage{listings}

\usepackage{letltxmacro}
\newcommand*{\SavedLstInline}{}
\LetLtxMacro\SavedLstInline\lstinline
\DeclareRobustCommand*{\lstinline}{
  \ifmmode
    \let\SavedBGroup\bgroup
    \def\bgroup{
      \let\bgroup\SavedBGroup
      \hbox\bgroup
    }
  \fi
  \SavedLstInline
}

\newcommand{\lstcommentcolor}[1]{\textcolor{darkorange}{#1}}
\newcommand{\lcc}[1]{\lstcommentcolor{#1}}

\usepackage[capitalize]{cleveref}
\crefname{listing}{Algorithm}{Algorithms}
\Crefname{listing}{Algorithm}{Algorithms}

\usepackage{graphicx}
\usepackage{fancyhdr}

\theoremstyle{definition}

\newcommand{\mynewtheorem}[3]{
\ifthenelse{\equal{#1}{theorem}}{
    \newtheorem{my#1}{#2}
}{
    \newtheorem{my#1}[mytheorem]{#2}
    \ifthenelse{\equal{#3}{}}{
        \crefname{my#1}{#2}{#2s}
    }{
        \crefname{my#1}{#2}{#3}
    }
}}

\mynewtheorem{theorem}{Theorem}{}
\mynewtheorem{lemma}{Lemma}{}
\mynewtheorem{corollary}{Corollary}{Corollaries}
\mynewtheorem{definition}{Definition}{}
\mynewtheorem{assumption}{Assumption}{}
\mynewtheorem{proposition}{Proposition}{}
\mynewtheorem{remark}{Remark}{}
\mynewtheorem{example}{Example}{}
\mynewtheorem{exercise}{Exercise}{}

\newenvironment{theorem}
  {\pushQED{\qed}\mytheorem}
  {\popQED\endmytheorem}
\newenvironment{lemma}
  {\pushQED{\qed}\mylemma}
  {\popQED\endmylemma}
\newenvironment{corollary}
  {\pushQED{\qed}\mycorollary}
  {\popQED\endmycorollary}
\newenvironment{proposition}
  {\pushQED{\qed}\myproposition}
  {\popQED\endmyproposition}

\newenvironment{definition}
  {\pushQED{\qed}\mydefinition}
  {\popQED\endmydefinition}

\newenvironment{remark}
  {\pushQED{\qed}\myremark}
  {\popQED\endmyremark}
\newenvironment{example}
  {\pushQED{\qed}\myexample}
  {\popQED\endmyexample}

\usepackage[framemethod=tikz]{mdframed}
\surroundwithmdframed[
  topline=false, rightline=false, bottomline=false, 
  innerleftmargin=10pt,
  leftmargin=-10pt,
  innerrightmargin=10pt,
  rightmargin=-10pt,
  innertopmargin=0,
  innerbottommargin=0,
]{proof}
\surroundwithmdframed[
  topline=false, rightline=false, bottomline=false,
  linewidth=.5pt,
  innerleftmargin=10pt,
  leftmargin=-10.5pt,
  innerrightmargin=10pt,
  rightmargin=-10pt,
  innertopmargin=-5pt,
  innerbottommargin=0,
]{example}
\surroundwithmdframed[
  topline=false, rightline=false, bottomline=false,
  linewidth=.5pt,
  innerleftmargin=10pt,
  leftmargin=-10.5pt,
  innerrightmargin=10pt,
  rightmargin=-10pt,
  innertopmargin=-5pt,
  innerbottommargin=0,
]{remark}
\newenvironment{proofnoqed}{\begin{proof}}{\end{proof}}

\definecolor{darkred}{rgb}{.5,0,0}
\definecolor{darkgreen}{rgb}{0,.5,0}
\definecolor{darkblue}{rgb}{0,0,.5}
\definecolor{darkorange}{rgb}{.8,.4,0}
\definecolor{violet}{rgb}{.5,0,1.}
\definecolor{darkcyan}{rgb}{0,.6,.7}
\ifdraft
\newcommand{\todo}[1]{\textcolor{darkorange}{(\emph{TODO: #1})}}
\newcommand{\comment}[1]{\textcolor{gray}{(\emph{#1})}}
\newcommand{\warning}[1]{\textcolor{red}{(\emph{WARNING: #1})}}
\newcommand{\quest}[1]{\textcolor{darkgreen}{(\emph{Q: #1})}}

\newcommand{\mh}[1]{\textcolor{gray}{(\textbf{MH:} \emph{#1})}}
\newcommand{\levi}[1]{\textcolor{gray}{(\textbf{LL:} \emph{#1})}}
\newcommand{\eser}[1]{\textcolor{gray}{(\textbf{EA:} \emph{#1})}}
\else
\newcommand{\todo}[1]{}
\newcommand{\comment}[1]{}
\newcommand{\warning}[1]{}
\newcommand{\quest}[1]{}

\newcommand{\mh}[1]{}
\newcommand{\levi}[1]{}
\newcommand{\eser}[1]{}
\fi

\newenvironment{eqitems}{
\begin{enumerate}[label=(\alph*),nosep]
}{
\end{enumerate}
}
\newenvironment{eqitems*}{
\begin{enumerate*}[label=(\alph*),nosep]
}{
\end{enumerate*}
}

\makeatletter
\newcommand{\mask}[2]{{\mathpalette\mask@{{#1}{#2}}}}
\newcommand{\mask@}[2]{\mask@@{#1}#2}
\newcommand{\mask@@}[3]{
  \settowidth{\dimen@}{$\m@th#1#2$}
  \makebox[\dimen@]{$\m@th#1#3$}
}
\makeatother

\newcommand{\makenamedref}[2]{
\expandafter\newcommand\csname#1name\endcsname{#2}
\expandafter\newcommand\csname#1label\endcsname{\label{#1_label}}
\expandafter\newcommand\csname#1ref\endcsname{\Cref{#1_label} (#2)}
}

\newcommand{\Proofof}{Proof of{}}
\newcommand{\ie}{\emph{i.e.}, }
\newcommand{\eg}{\emph{e.g.}, }

\DeclareMathOperator*{\argmin}{argmin}

\newcommand{\Naturals}{{\mathbb N}_1}

\newcommand{\Reals}{{\mathbb R}}

\newcommand{\indicator}[1]{\left\llbracket{#1}\right\rrbracket}

\newcommand{\D}{\mathrm{d}}

\newcommand{\floor}[1]{\left\lfloor#1\right\rfloor}
\newcommand{\ceil}[1]{\left\lceil#1\right\rceil}

\newcommand{\eps}{\varepsilon}

\newcommand{\nodeset}{\mathcal{N}}
\newcommand{\leafset}{\mathcal{L}}
\newcommand{\children}{\mathcal{C}}

\newcommand{\nodesettheta}{\nodeset_\theta}
\newcommand{\nodesetthetaplus}{\nodeset_{\theta}^{+}}

\newcommand{\parent}{\mathrm{par}}
\newcommand{\anc}{\mathrm{anc}}
\newcommand{\ancn}{\mathrm{anc}_+}
\newcommand{\desc}{\mathrm{desc}}
\newcommand{\descn}{\mathrm{desc}_+}
\newcommand{\pol}{\pi}

\newcommand{\nup}{\overline{n}}
\newcommand{\ndn}{\underline{n}}

\newcommand{\costargs}[2]{(#1;#2)}
\newcommand{\slend}{\lambda}
\newcommand{\dop}{\tfrac{d}{\pol}}
\newcommand{\dopa}[2]{\dop\costargs{#1}{#2}}
\newcommand{\sop}{\tfrac{\lambda}{\pol}}
\newcommand{\sopa}[2]{\sop\costargs{#1}{#2}}

\newcommand{\tpr}{w}
\newcommand{\ttpr}{\tilde\tpr}
\newcommand{\Tprm}[1]{
    \tpr{\raisebox{-1pt}{\ensuremath{{}_{#1}}}}}
\newcommand{\TTprm}[1]{
    \ttpr{\raisebox{-1pt}{\ensuremath{{}_{#1}}}}}

\newcommand{\queue}{Q}
\newcommand{\Cnm}{\tilde c}
\newcommand{\Crlts}{c^{\mathrm{r}}}
\newcommand{\Cmax}{c^{\mathrm{max}}}

\newcommand{\clueset}{\nodeset_q}

\newcommand{\rootlts}{\ensuremath{\sqrt{\mathrm{\textnormal{\textsc{lts}}}}}}

\newcommand{\RED}[1]{\textcolor{red}{#1}}
\newcommand{\BLUE}[1]{\textcolor{blue}{#1}}

\usepackage{tikz}

\usetikzlibrary{positioning,calc}
\usetikzlibrary{trees}
\usetikzlibrary{patterns}
\usetikzlibrary{patterns.meta}

\usepackage{forest}

\title{Exponential Speedups by Rerooting Levin Tree Search}
\author{
Laurent Orseau$^1$
\and
Marcus Hutter$^1$\and
Levi H. S. Lelis$^{2}$\\
}
\date{
$^1$Google DeepMind\\
$^2$Department of Computing Science, University of Alberta, Canada\\
and Alberta Machine Intelligence Institute (Amii), Canada\\[1ex]
\{lorseau,mhutter\}@google.com,
levi.lelis@ualberta.ca
}

\newenvironment{keywords}{\centerline{\bf\small Keywords}\begin{quote}\small}{\par\end{quote}\vskip 1ex}

\begin{document}

\maketitle

\begin{abstract}
Levin Tree Search (LTS) \citep{orseau2018single} is a search algorithm for deterministic environments that uses a user-specified \emph{policy} to guide the search.
It comes with a formal guarantee on the number of search steps (node visits) for finding a solution node that depends on the quality of the policy.
In this paper, we introduce a new algorithm, called \rootlts{} (pronounce root-LTS), which implicitly 
starts an LTS search rooted at every node of the search tree.
Each LTS search is assigned a \emph{rerooting weight} by a \emph{rerooter}, and the search effort is shared between all LTS searches proportionally to their weights.
The rerooting mechanism implicitly decomposes the search space into subtasks, leading to significant speedups.
We prove that the number of node visits that \rootlts{} takes
is competitive with the best decomposition into subtasks, at the price of a factor that relates to the uncertainty of the rerooter.
If LTS takes time $T$, in the best case with $q$ rerooting points, \rootlts{} only takes time $O(q\sqrt[q]{T})$.
Like the policy, we expect that the rerooter can be learnt from data, making \rootlts{} applicable to a wide range of domains.
\ifarxiv\vspace{5ex}\def\contentsname{\centering\normalsize Contents}\setcounter{tocdepth}{1}
{\parskip=-2.7ex\tableofcontents}\fi
\end{abstract}


\ifarxiv\clearpage\fi
\section{Introduction}\label{sec:Intro}

We are interested in tree search algorithms for deterministic domains.
Tree search algorithms such as all variants of best-first search --- including A*~\citep{hart1968formal}, Weighted-A* (WA*), and Greedy Best-First Search (GBFS)~\citep{doran1966gbfs} --- and variants of MCTS --- such as UCT~\citep{kocsis2006uct}, AlphaGo, AlphaZero and other variants~\citep{silver2016alphago,silver2017mastering,silver2017zero} --- explore the search tree starting at the root and can visit a node only if its parent has been visited first.
These algorithms are often guided by some side information, such as with cost-to-go heuristic function for A*, WA* and GBFS, a reward/value function for UCT and AlphaZero, or a policy for AlphaZero, Levin Tree Search (LTS)~\citep{orseau2018single}, and Policy-Guided Heuristic Search~\citep{orseau2021policy}.

Such algorithms also sometimes come with different types of guarantees:
A* and WA*, with an admissible heuristic function --- \ie a function that never overestimates the optimal cost-to-go --- are guaranteed to return a solution 
that is cost-optimal (A*) or bounded-suboptimal (WA*), while UCT and AlphaZero are guaranteed to (eventually) have low regret in terms of cumulative reward during the search.
LTS is guaranteed to return a solution within a number of node visits that depends on the quality of its policy.
In this paper, we consider the latter type of guarantee,
on the efficiency of the search process depending on the quality of the side information.

To explain the main concepts of this paper, let us consider a kind of side information we call \emph{clues}: some nodes are clue nodes, and a node can be known to be a clue node only when reaching it.
A clue may be \emph{helpful} if it is on the path toward a solution node,
or \emph{misleading} otherwise.
The following example describes a minimalistic clue environment.

\begin{figure}[tb]
    \centering
    \newcommand{\mytriangle}[4]{
\draw[fill=#4,color=#4] let \p1  = #1 in let \p2 = #2 in (\x1, \y1) 
-- (\x1 - #3cm,\y2) -- (\x1 + #3cm, \y2) -- cycle;
}
\begin{tikzpicture}[
        scale=0.6,
        mynode/.style={fill=white,draw=black,circle,scale=.7},
    ]
    \begin{scope}[xshift=0cm]
      \node (Top) at (0, 0) {};
      \node (Na) at (-3.5, -2) {};
      \node (Nb) at (-1., -2) {};
      \node (Nc) at (1.8, -2) {};
      \node (Nd) at (4, -2) {};
      \node (NT) at (1.2, -4) {};
        \mytriangle{(Top)}{(NT)}{8}{lightgray}
        \mytriangle{(Top)}{(Na)}{4}{darkgray}
        \mytriangle{(Na)}{(NT)}{1}{darkgray}
        \mytriangle{(Nb)}{(NT)}{1}{darkgray}
        \mytriangle{(Nc)}{(NT)}{1}{darkgray}
        \mytriangle{(Nd)}{(NT)}{1}{darkgray}
        \draw[white, line width=0.5mm] (Top) -- (Nc) -- (NT);
        \draw (Top) node[mynode] {$n_1$};
        \draw (Na) node[mynode] {$n_a$};
        \draw (Nb) node[mynode] {$n_b$};
        \draw (Nc) node[mynode] {$n_c$};
        \draw (Nd) node[mynode] {$n_d$};
        \draw (NT) node[mynode] {$n^*$};
        % \node at (Top) [right,xshift=1ex] {$n_1$};
        % \node at (Na) [right,xshift=1ex, yshift=-1ex] {$n^a$};
        % \node at (Nb) [right,xshift=1ex, yshift=-1ex] {$n^b$};
        % \node at (Nc) [right,xshift=1ex, yshift=-1ex] {$n^c$};
        % \node at (Nd) [right,xshift=1ex, yshift=-1ex] {$n^d$};
        % \node at (NT) [right,xshift=1ex, yshift=-1ex] {$n^*$};
    \end{scope}
\end{tikzpicture}
    \caption{A schematic representation of the binary tree of \Cref{ex:1000clues} with four clue nodes $n_a, n_b, n_c, n_d$ at depth 50,
    and the solution node $n^*$ at depth 100.
    }
    \label{fig:clue_env_triangles}
\end{figure}
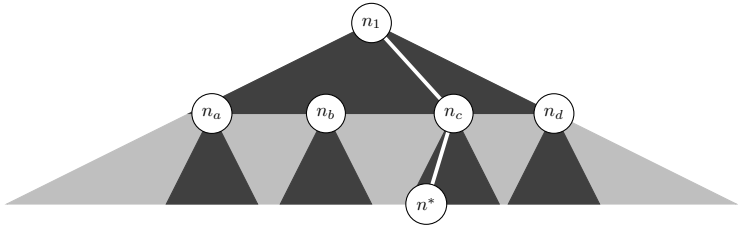

\begin{example}[1000-Clues]\label{ex:1000clues}
The search space is a perfect binary tree of depth 100, and the single solution node is one of the leaves.
Without any further information, no tree search algorithm can do better than searching through the $2^{100}$ leaf nodes in the worst case. 
Fortunately, at depth 50 there are 1000 random \emph{clue nodes}, and the solution node is a descendant of one of these clue nodes.
That is, only one clue node is helpful, and 999 clue nodes are misleading.
See \Cref{fig:clue_env_triangles}.
The search algorithm can discover whether a node $n$ is a clue node only upon visiting it.\footnotemark\ 
How can we leverage this information to search faster?
A simple algorithm first finds all the 1000 clue nodes, which takes about $2^{50}$ steps, then starts a new search under each clue node for about $2^{50}$ steps each until finding the solution node.
This algorithm would visit about $1001\times 2^{51}$ nodes in the worst case,
which is far less than the $2^{100}$ steps required on average when the clues are not used.
\end{example}
\footnotetext{Clues could be observed also upon node generation, that is, when the parent is visited --- see \cref{eq:queue_update}.}

In \Cref{apdx:admissible}, we show that A* and WA* cannot make use of clues in general, while retaining their bounded-suboptimality guarantees.
Policy-guided algorithms like LTS use a probability distribution over the available actions at any node of the tree:
LTS is a best-first search with visits the nodes in increasing order of their cost $n\mapsto d(n) /\pol(n)$ where $d(n)$ is the depth of a node $n$,
and $\pol(n)$ is the product of the edge (or `action') probabilities from the root to $n$.
On \Cref{ex:1000clues}, it can be shown that LTS needs to visit at least $2^{99}$ nodes in the worst case until visiting $n^*$ --- see \Cref{apdx:lts_clues}.
Reward-based algorithms such as AlphaZero are greedy in the sense that they spend most of their search time in the subtree rooted in the first high-reward node they visit.
This greedy behavior can lead to taking double-exponential time (with the solution depth) to recover from misleading rewards/clues~\citep{coquelin2007bandit,orseau2024superexponential}
--- see \Cref{apdx:mcts}.

In this paper, we propose the \rootlts{} algorithm (pronounced root-LTS), which solves the kind of clue environments described in \Cref{ex:1000clues} and much more.
\rootlts{} uses both a policy and a new guiding construct called \emph{rerooter}.
The rerooter assigns a \emph{rerooting weight} to each visited node.
This weight can depend on whatever information is available when visiting the node, including whether the node is a clue node, or how many clue nodes have been visited up to this node, or any other feature of the current and previously-visited nodes. 
Intuitively, the rerooting weight of a node informs \rootlts{} what fraction of its search time should be devoted to the tree rooted in this node. 
\rootlts{} comes with formal guarantees on the number of node visits required to reach a solution node, depending on the quality of the policy and of the rerooter. 
Like the LTS policy, we expect that the rerooter may be learnt from data or designed by the user, but this is beyond the scope of this paper.

While \rootlts{} does not specifically require clues since only the rerooting weights matter, clues are still a useful concept both for intuition and perhaps to design rerooters.
We envision that clues could be given whenever the search is `on the right path'.
This is related to rewards and shaping rewards in reinforcement learning~\citep{ng1999policy,sutton1998reinforcement},
landmarks in classical planning~\citep{hoffmann2004landmarks}.
In automated theorem proving~\citep{loveland2016automated}, a clue may be given when a hopefully-helpful lemma is found.
In constraint satisfaction programming~\citep{tsang2014foundations}, a clue could be given when 
some difficult constraints are satisfied.
More generally, a clue could be given whenever substantial progress is made on some scoring function, or when a bottleneck has been passed, or when a subtask has been solved.
Clues may not all have the same meaning either, and rerooting weights need not be equal for all clues.

After introducing notation (\Cref{sec:notation}),
we present in \Cref{sec:lower_bound} a generalization of \Cref{ex:1000clues} that will serve both as a running example and for deriving lower bounds.
\Cref{sec:self-counting} introduces self-counting cost functions
and how to compose them, and proposes the self-counting cost function $\sop$
as a replacement for $\dop$ in LTS with tighter bounds.
In \Cref{sec:rootlts}, we build the best-first-search cost function of \rootlts{}
and give its main guarantee as well as simplified ones.
Finally, we make our algorithm robust to clue overload (\Cref{sec:robustness}).

\section{Notation and Background}\label{sec:notation}

A table of notation can be found in \Cref{apdx:table_notation}.
The set of nodes is $\nodeset$.
The set of children of a node $n$ is $\children(n)$.
Each node $n$ has either one parent $\parent(n)$ except for the root node $n_1$ which has no parent.
The set of descendants of a node $n$ is $\desc(n)$ (the transitive closure of $\children(\cdot))$, and we define $\descn(n) = \desc(n)\cup\{n\}$.
Similarly, the set of ancestors of $n$ is $\anc(n)$, and we define $\ancn(n) = \anc(n)\cup\{n\}$.
For two nodes $n$ and $\ndn$,
we write $n \prec \ndn$ for $n\in \anc(\ndn)$ and 
$n\preceq \ndn$ for $n\in\ancn(\ndn)$.
The depth of a node $n$ is $d(n) = |\anc(n)|$, hence $d(n_1)=0$.

A policy $\pol:\nodeset\times\nodeset\to [0,1]$,
written $\pol(\cdot| \cdot)$,
is such that for all $n \not\preceq \dot n: \pol(\dot n| n) = 0$,
and for each node $n$,
$\sum_{\ndn\in\children(n)} \pol(\ndn| n) \leq 1$;
the policy is called \emph{proper} if this holds with equality.
We define the \emph{path probability} recursively:
for all $n$, $\pol(n| n) = 1$, and 
for all $\nup\preceq n\prec \ndn$
with $\ndn\in\children(n): \pol(\ndn| \nup) = \pol(\ndn| n) \pol(n| \nup)$.
We write $\pol(n) = \pol(n| n_1)$.
We say that a policy is \emph{uniform} if all children of a node have the same conditional probability: $\pol(\ndn| n) = 1/|\children(n)|$ if $\ndn\in\children(n)$.
The policy $\pol$ is assumed to be given by the user, and may be learnt (\eg \citet{orseau2021policy,orseau2023ltscm}).

A \emph{cost function} is a function from the set of nodes $\nodeset$ to the reals $\Reals$.
A cost function $c:\nodeset\to\Reals$ is \emph{monotone}
if it is monotone non-decreasing from parent to child, or equivalently
$c(n) \leq c(\ndn)$ for all $n \prec \ndn$.
Given a (possibly non-monotone) cost function $c:\nodeset\to\Reals$, a \emph{best-first search} (BFS)
\citep{pearl1984heuristics}
is an enumeration $n_1, n_2, \dots$ of the nodes using a priority queue $\queue$, initialized with $\queue_1 = \{n_1\}$.
For each step $t$, define:
\begin{align}
    n_t &\in\argmin_{n\in\queue_t} c(n) \,, \text{ with ties broken arbitrarily} 
    \label{eq:visited_node}\\
    \queue_{t+1} &= (\queue_t \setminus\{n_t\})\cup \children(n_t)\,.
    \label{eq:queue_update}
\end{align}
The node $n_t$ is the \emph{visited} node at step $t$ of the BFS.
With a \emph{monotone} cost function $c$, it is well-known that BFS ensures that the number of visited nodes $t$ when the node $n_t$ is visited is bounded by $t = |\{n_1,\dots n_t\}|\leq |\{n\in\descn(n_1): c(n)\leq c(n_t)\}|$.

We also assume that the user provides a \emph{rerooter}
$\tpr:\nodeset\to[0,\infty)$.
At step $t$ of a BFS enumeration, when visiting the node $n_t$,
the rerooter assigns a \emph{rerooting weight} $\tpr_t$
that can depend on all the information available up to and including step $t$.
During such a BFS enumeration, the \emph{cumulative rerooting weight} at any step $T$
is $\Tprm{\leq T} = \sum_{t\leq T} \tpr_t$.

\begin{remark}[No stopping]\label{rmk:no_stopping}
For the sake of generality, in this paper we do not use a stopping criterion for BFS. Our upper bounds hold for every node visited before the algorithm stops, irrespective of the stopping criterion used.
\end{remark}

\section{Lower Bound}\label{sec:lower_bound}

First we prove a general lower bound on the number of node visits that any search algorithm, randomized or deterministic,
using any kind of heuristic guide,
must perform before visiting a solution node in the presence of clues.
We will also use this set of environments as a running example and to show that some improvements to \rootlts{} are impossible in general.
In the following result, for simplicity we assume that the algorithm can test whether a node is the solution node $n^*$ only upon visiting it.

\begin{theorem}\label{thm:lower_bound}
Consider an infinite perfect binary tree.
Choose $a, q\in\Naturals$.
Let $\clueset\subseteq\nodeset$ be the set of clue nodes,
where $q=|\clueset|$ is the number of clue nodes.
The root is a clue node.
For every other clue node $n$,
there exists a clue ancestor $\nup$ at relative depth at most $a$:
$1\leq d(n) - d(\nup) \leq a$.
Optionally, communicate $a, q, \clueset$ to the search algorithm.
Let $\nodeset_{q,a} = \bigcup_{n\in\nodeset_q}\{\ndn: n\preceq \ndn, d(\ndn) - d(n) \leq a \}$ be the set of nodes of depth at most $a$ relative to any clue node of $\nodeset_q$.
The solution node $n^*$ is chosen uniformly at random among $\nodeset_{q,a}$.

Then, for all such $a, q, \clueset$,
every search algorithm must visit at least $(q+1)2^{a-1}$ nodes on average before visiting $n^*$ --- and thus also in the worst case.
\end{theorem}

Note that instead of communicating $\clueset$ to the search algorithm, a less informed alternative is to communicate a clue membership function
$n\mapsto \indicator{n\in\clueset}$
which can only be applied to generated nodes (see \cref{eq:queue_update}) --- the result still holds in this case.

\begin{proof}
For a given set $\clueset$, let us place the clues 
in the tree, one by one, in increasing order of their depths.
First, the root is a clue node.
Since the solution node $n^*$ can descend from any clue node at relative depth at most $a$,
there are for now $2^{a+1}-1$ possible placements.
The second clue node is at depth at most $a$ relative to its closest ancestor clue node (the root, in this case).
The solution node $n^*$ can now also be placed among any of the closest $2^{a+1}-1$ descendants of the second clue node.
But, among these, all (and only) the nodes at depth at most $a$ relative to the first clue node have already been counted.
This overlap is the largest when the second clue node is a child of
the first clue node (rather than a more distant descendant), in which case the cardinality of the overlap is $2^a-1$.
Hence the number of \emph{new} places for $n^*$ is at least $2^{a+1}-1 - (2^a-1) = 2^a$, and the cumulative number of places at this stage is 
at least $2^{a+1}-1 + 2^a$.
After repeating the process until all $q$ clues are placed, the number of possible places $|\nodeset_{q,a}|$ for $n^*$ is thus at least $2^{a+1} - 1 + (q-1)2^a = (q+1)2^a -1$.
(Ordering the clue nodes by their depth ensures that when placing a new clue node $n$, no descendant of $n$ is already a clue node, and thus the overlap to consider is only with respect to its closest clue ancestor.)
Therefore,
by randomizing the location of $n^*$,
any deterministic algorithm needs at least $(q+1)2^a/2$ node visits on average to visit $n^*$,
and by a standard argument this holds also for any randomized algorithm.
\comment{More precisely, if $A(s)$ is a randomized algorithm instantiated with random seed $s$ (which makes it deterministic), then 
$A(s)$ takes on average $(q+1)2^a/2$ node visits, and taking the expectation on $s$ doesn't change this.}
\end{proof}

\Cref{thm:lower_bound} means that, for the considered set of environments,
if there are $m$ clue nodes on the path from the root to the solution node $n^*$
--- these are clue nodes that any tree search algorithm must visit before visiting $n^*$ --- it is not possible to achieve an upper bound of
$O(m2^a)$ node visits (on average or in the worst case) if $m = o(q)$.
For example, if every path from the root has at most $\sqrt{q}$ clue nodes,
then $m \leq \sqrt{q}$, and so an algorithm with an upper bound of $O(m2^a) = O(\sqrt{q}2^a)$ does not exist as this contradicts \Cref{thm:lower_bound}.
Similarly, if it happens that the clues are at depths at most $b < a$ relative to their closest clue ancestors (and the solution node is still at relative depth $a$ to its closest clue ancestor), it is not possible to achieve an upper bound of $O(q2^b)$ in general, as taking $b = a-\ln a$ would contradict \Cref{thm:lower_bound}, even if $b$ is communicated to the search algorithm (since \Cref{thm:lower_bound} allows for $\clueset$ to be communicated).

\section{Self-Counting Cost Functions}\label{sec:self-counting}

In this section, we first define self-counting cost functions and show their relation to the BFS
steps at which nodes are visited.
Next, as a side contribution, we improve the self-counting cost function of LTS.
Then, we show how to compose self-counting cost functions into a single one.

\begin{definition}[Self-counting cost function]\label{def:self_counting} 
A cost function $c:\nodeset\to\Reals$ is said to be \emph{self-counting} if, for all $\theta \geq 0$,
the number of nodes of cost at most $\theta$ is itself at most $\theta$:
\begin{align*}
  |\{n\in\nodeset: c(n) \leq \theta\}| \leq \theta\,.
  &\qedhere
\end{align*}
\end{definition}
Note that non-monotone cost functions may still be self-counting.

\citet{orseau2018single,orseau2023ltscm}
use the monotone cost function $\dop(n) = d(n) / \pol(n)$ to guide the BFS of Levin Tree Search.
They prove that $|\{n: \dop(n) \leq \theta\}| \leq 1+\theta$ for all $\theta \geq 0$.
But because BFS is invariant to cost translation,
BFS with the cost function $\dop(\cdot)$ is equivalent to BFS with the cost function $1+\dop(\cdot)$.
Since $|\{n: 1+\dop(n) \leq 1+\theta\}| \leq 1+\theta$,
which is equivalent to $|\{n: 1+\dop(n) \leq \theta\}| \leq \theta$,
the cost function $1+\dop(\cdot)$ is a self-counting cost function.

The following result strongly links the number of BFS steps with the cost of a monotone self-counting cost function.

\newcommand{\lemtleqcnt}{
A monotone cost function $c$ is self-counting if and only if $t \leq c(n_t)$ for all $t\in\Naturals$, where
$n_t$ is visited at step $t$ of the BFS with the cost function $c$.}
\begin{lemma}\label{lem:t_leq_cnt}
\lemtleqcnt
\end{lemma}
\begin{proof}
Assume $c$ is self-counting.
Since $c$ is monotone, BFS enumerates the nodes in order of increasing costs, 
so for all $t$ we have
\begin{equation*}
    t ~=~ |\{n_1, \dots n_t\}|
    ~\leq~ |\{n: c(n) \leq c(n_t)\}| 
    ~\leq~ c(n_t)
\end{equation*}
as required, where the last inequality follows from the self-counting property.

For the other direction,
we assume that $t \leq c(n_t)$ for all $t$.
For any $\theta$, 
choose the largest $\tau$ such that $c(n_\tau)\leq \theta$.
Since $c$ is monotone, all nodes visited before $\tau$ have cost at most $\theta$, and all nodes visited after $\tau$ have cost strictly more than $\theta$,
that is,
\begin{equation*}
    |\{n:c(n) \leq \theta\}|
    ~=~ |\{n_1,\dots n_\tau\}|
    ~=~ \tau
    ~\leq~ c(n_\tau)
    ~\leq~ \theta
\end{equation*}
and thus $c$ is self-counting.
\end{proof}

Hence, for example, since the cost function $1+\dop(n)$ is self-counting and monotone,
the number of steps $T$ before BFS (with this cost function) visits the node $n_T$ is bounded by $T\leq 1+\dop(n_T)$ --- as was shown by \citet{orseau2018single}.
Additional examples and remarks can be found in \Cref{apdx:self_counting}.

\subsection{Slenderness Cost Function}\label{sec:slend}

Now we define a variant of the cost function $1+\dop(\cdot)$ with tighter guarantees.
Indeed, as a self-counting cost function, $1+\dop(\cdot)$ is a little loose, as shown by the following example.

\begin{wrapfigure}{R}{0.3\textwidth}
    \centering
\scalebox{0.7}{
\begin{forest}
for tree={
 grow'=east,
 l=1.2cm,
 s sep=0.6cm,
 minimum width=2em,
 draw,circle,
}
[$n_1$,
 [,edge label={node[midway,above]{1}},
     [\ldots, draw=none,edge label={node[midway,above]{1}},
         [$n_a$,edge label={node[midway,above]{1}}
             [$n_b$, edge label={node[midway,above]{$\frac{1}{2}$}}]
             [$n_c$, edge label={node[midway,below]{$\frac{1}{2}$}}]]]]]
\end{forest}
}
    \caption{The tree of \Cref{ex:loose_dop,ex:tight_sop}.
    }
    \label{fig:loose_dop}
\end{wrapfigure}

\begin{example}[$\dop$ double counts]\label{ex:loose_dop}
Consider a chain of nodes where each node has exactly one child with conditional probability 1 --- see \Cref{fig:loose_dop}.
Since $1+\dop$ is self-counting,
the number of nodes of cost at most the cost of $n_a$ is 
$|\{n: 1+\dop(n) \leq 1+\dop(n_a)\}| \leq 1+\dop(n_a)= 1+d(n_a)$, and indeed
there are exactly $d(n_a)+1$ such nodes.
Now suppose that $n_a$ has 2 children $n_b$ and $n_c$, each with conditional probability 1/2.
Then $1+\dop(n_b) = 1+\dop(n_c) = 2d(n_a)+3$,
which is indeed also an upper bound on the number of nodes of cost at most the cost of $n_b$.
However, this upper bound is loose because there are in fact only $d(n_a)+3$ such nodes.
This is because the ancestors of $n_b$
and the ancestors of $n_c$ are counted separately, leading to double-counting.
\end{example}

Indeed, the factor $d(n)$ appears in the ratio between the upper and lower bounds \citep[Theorem 2]{orseau2023ltscm}.
We develop a tighter self-counting cost function $\sop(n)$ that avoids this double counting.
The cost function $\sop(n)$ is based on the quantity $\slend(n)$ which counts what fraction of its ancestors the node $n$ is `responsible' for.
If a node $n$ holds a share $\slend(n)$ of its ancestors (including itself),
then a child $\ndn$ of $n$ holds a share $\slend(n)\pol(\ndn| n)$
of the ancestors of $n$, and thus a share
$\slend(n)\pol(\ndn| n) + 1$ of the ancestors of $\ndn$.
We call $\slend:\nodeset\to[1,\infty)$ the \emph{slenderness} of a node,
which, for a given policy $\pol$,
is defined as $\slend(n_1)\! =\! 1$ and for each $n$,
\begin{align}    
    \forall \ndn\in\children(n):\quad \slend(\ndn) \ =\  \slend(n)\,\pol(\ndn| n) + 1\,. \label{eq:slend}
\end{align}
It follows that for every node $n$, $1 \leq \slend(n) \leq 1+d(n)$.
Then, for a finite tree,
the sum of the slenderness of the leaves should be a tighter count of the number of nodes in the tree than naively summing the depths of the leaves.
From this we define a tighter self-counting cost function.

Define the \emph{slenderness cost function} $\sop:\nodeset\to[1,\infty)$, for all $n$:
\begin{align}\label{eq:sop}
    \sop(n)\ =\ \frac{\slend(n)}{\pol(n)}\ =\  \sop(\parent(n)) + \frac1{\pol(n)}
\end{align}
which shows that $\sop$ is a monotone cost function.
We can deduce the useful formula:
\begin{align}\label{eq:sop_anc}
    \sop(n)\ =\ \sum_{\nup\preceq n}\frac{1}{\pol(\nup)}\,.
\end{align}
Since $\pol(n_1) =1$ and $\pol(\nup)\geq \pol(n)$ and $d(n) = |\anc(n)|$,
this implies that $\sop(n) \leq 1+\dop(n)$.
The slenderness cost function is self-counting, 
with a tight lower bound 
if the policy is proper:
\begin{align}\label{eq:sop_bounds}
    \text{for all }\theta\geq 0:\quad\quad \frac{\theta-1}{B} ~<~ |\{n: \sop(n) \leq \theta \}| ~\leq~ \theta
\end{align}
where $B$ is the average branching factor in the tree $\{n: \sop(n) \leq \theta \}$.
The proof is in \Cref{apdx:slenderness}.
The ratio between the upper and lower bounds for $\sop$ is tighter by a factor $d(\cdot)$ than the one obtained by \citet[Theorem 2]{orseau2023ltscm} for $\dop$.

\begin{example}[\Cref{ex:loose_dop} continued]\label{ex:tight_sop}
For the node $n_a$, $\sop(n_a) = d(n_a)+1 = 1+\dop(n_a)$.
For the nodes $n_b$ and $n_c$,
$\sop(n_b) = \sop(n_c) = \sop(n_a) + 1/(1/2) = d(n_a)+3$ which, this time, is exactly the number of nodes of cost at most the cost of $n_b$.
\end{example}

More examples can be found in \Cref{apdx:slenderness}.
From now on, we assume that LTS uses the cost function $\sop(\cdot)$ instead of $1+\dop(\cdot)$.

\paragraph{Rooted slenderness cost function.}
We will need to consider the $\sop$-cost function rooted in some node $n_a\in\descn(n_1)$,
hence we define the \emph{rooted $\sop$-cost function} by generalizing \cref{eq:sop_anc} to
\begin{align}\label{eq:sop_anc_root}
    \sopa{n}{n_a} = \sum_{n_a \preceq \nup \preceq n}\frac1{\pol(\nup|n_a)}\,,
\end{align}
and $\sopa{n}{n_a} = \infty$ if $n_a\not\preceq n$.
When taking $n_a=n_1$, \cref{eq:sop_anc_root} reduces to \cref{eq:sop_anc}.
Note that while $\sopa{\cdot}{n_a}$ is still self-counting,
\footnote{All the proofs can be readily adapted at the expense of heavier notation.}
it is \emph{not} monotone since the costs of the ancestors of the root $n_a$ are infinite --- but it is monotone on the descendants of $n_a$.

\subsection{Composing Self-Counting Cost Functions}\label{sec:compose_self-counting}

We show how self-counting cost functions can be composed,
which is a central idea of our algorithm.
But first, suppose that we have $N$ algorithms that all try to solve the same problem in different ways,
and they all run on the same CPU.
We want to share the computation steps non-uniformly between the algorithms, and we have weights
 $\tpr_1, \tpr_2, \dots \tpr_N$ that sum to 1.
After a total of $T$ computation steps, each algorithm $i$ has been assigned at least $\floor{\tpr_i T}$ computation steps.
Now, if algorithm $i$ needs $\tau_i$ steps to find a solution,
then this happens when the total $T$ is such that $\floor{\tpr_i T} = \tau_i$.
This implies $T \leq (\tau_i+1)/\tpr_i$.
Since this holds for all $i$, a solution is found after $T$ steps where
\footnote{Related ideas appear at least as far back as \citet{levin1973search}. See also \citet{li2019introduction}.}
\begin{align}\label{eq:T_leq_min_tau/w}
    T \leq \min_{i\in[N]}\frac{\tau_i+1}{\tpr_i}\,.
\end{align}
We use a similar idea to compose a set of \emph{base} self-counting cost functions into a \emph{single} self-counting cost function to be used within a \emph{single} instance of BFS.

\begin{lemma}[Composing self-counting cost functions]\label{lem:compose_sccf}
Assume we have $N$ base self-counting cost functions $\{c_i\}_{i\in[N]}$.
For a weighting $\tpr$ (with $\sum_{i\in[N]} \tpr_i \leq 1$, $\forall i: \tpr_i \geq 0$),
let 
\begin{align}\label{eq:compose_cost}
    c(n) = \min_{i\in [N]} \frac{c_i(n)}{\tpr_i}\,.
\end{align}
then $c$ is a self-counting cost function.
\end{lemma}
Note that \Cref{lem:compose_sccf} does not assume or imply monotonicity.
\comment{Though if all $c_i$ are monotone, then $c$ is monotone too.}
\begin{proof}
From the definition of $c(n)$,
for all $\theta\geq 0$, the number of nodes of cost at most $\theta$ is 
\begin{multline}
    |\{n: c(n)\leq \theta\}| =
    \left|\left\{n:\min_{i\in [N]} \frac{c_i(n)}{\tpr_i} \leq \theta\right\}\right|
    = \left|\bigcup_{i\in[N]}\left\{n: \frac{c_i(n)}{\tpr_i} \leq \theta\right\}\right| \\
    \leq 
    \sum_{i\in[N]}
    \left|\left\{n: \frac{c_i(n)}{\tpr_i} \leq \theta\right\}\right|
    = 
    \sum_{i\in[N]}
    |\{n: c_i(n) \leq \tpr_i\theta\}|
    \leq \sum_{i\in[N]} \tpr_i\theta \leq \theta\,, \notag
\end{multline}
where the first inequality is a union bound,
the second inequality is because each $c_i$ is self-counting,
and the last inequality is by definition of $\tpr$.
\end{proof}

From \Cref{lem:compose_sccf,lem:t_leq_cnt},
if $n_T$ is the node visited at step $T $ of BFS with a \emph{monotone} compound cost function $c$, then
\begin{align}\label{eq:compose_t_leq_cnt}
    T \leq c(n_T) = \min_{i\in[N]} \frac{c_i(n_T)}{\tpr_i}\,.
\end{align}
Hence, \cref{eq:compose_cost} can be seen as a way to share the best-first search time between the different self-counting cost functions,
while \cref{eq:compose_t_leq_cnt} for self-counting cost functions is analogous to \cref{eq:T_leq_min_tau/w} for programs.
Additional remarks can be found in \Cref{apdx:self_counting}.

\section{The \texorpdfstring{\rootlts}{Root-LTS} Algorithm and its Guarantee}\label{sec:rootlts}

The idea of \rootlts{} is to implicitly start an instance of LTS (using $\sop$) rooted at every visited node $n_1, n_2,\dots$ of the search and compose these LTS instances with weights $\tpr_1,\tpr_2,\dots$ 
into a single cost function to be used with a single instance of BFS
as in \Cref{lem:compose_sccf}.
Let us make a first attempt:
For each $t\in\Naturals$, we define the base cost function $\Cnm_t(n) = \sopa{n}{n_t}$, and these are composed via $\Cnm(n) = \min_{t\in\Naturals} \Cnm_t(n)/\tpr_t$. 
Since $\Cnm_t(n) = \infty$ for $n_t \not\preceq n$, this can be simplified to
\begin{align}\label{eq:Cnm}
    \Cnm(n) = \min_{n_t \preceq n}\frac1{\tpr_t}\sopa{n}{n_t}\,.
\end{align}
Note that the index $t$ of the base cost function $\Cnm_t$ 
corresponds to the step $t$ at which the node $n_t$ is visited during the BFS with the cost function $\Cnm$.

Unfortunately, the base cost function $\sopa{\cdot}{n_t}$ is monotone only on the descendants of $n_t$, but is \emph{not} monotone in general on the descendants of $n_1$.
Hence, while $\Cnm$ is self-counting by \Cref{lem:compose_sccf},
it is itself \emph{not} monotone either.
This prevents using \cref{eq:compose_t_leq_cnt} to derive a straightforward bound.
But this issue is easy to fix by defining
$
    \Cmax(n) = \max_{\nup\preceq n} \Cnm(\nup)
$
which is monotone by design, and preserves the self-counting property.
\footnote{
Since $\Cmax(n) \geq \Cnm(n)$ for all $n$,
then for all $\theta \geq 0$,
$|\{n:\Cmax(n)\leq \theta\}| \leq |\{n:\Cnm(n)\leq \theta\}| \leq \theta$.}
Then,
assuming that the weights $\{\tpr_t\}_t$ sum to at most 1,
from \cref{eq:compose_t_leq_cnt},
for the node $n_T$ visited at step $T$ of BFS with the cost function $\Cmax$,
\begin{align}\label{eq:Cmax_visits}
    T~\leq~ \Cmax(n_T) ~=~ \max_{n\preceq n_T} \Cnm(n)
    ~=~  \max_{n\preceq n_T} \min_{n_t\preceq n} \frac1{\tpr_t}\sopa{n}{n_t}\,.
\end{align}

\begin{example}[Uniform clue weights]\label{ex:normalized_clue_bound}
Consider the environments of \Cref{thm:lower_bound} (\Cref{sec:lower_bound}).
The policy is uniform.
Assuming that $q$ is known (this assumption will be relaxed later),
set $\tpr_t = 1/q$ if the node $n_t$ is a clue node, $\tpr_t = 0$ otherwise.
For every node $n$ on the path from $n_1$ to $n_T$,
there exists a clue ancestor $n_s$ of $n$ at relative depth at most $a$.
Thus, $\Cnm(n) \leq \sopa{n}{n_s} /\tpr_s \leq  q(1+2+4+\dots\ 2^a) = q(2^{a+1}-1)$ (see \cref{eq:sop_anc_root}).
Then, \Cref{eq:Cmax_visits} tells us that if $n^*$ is visited at step $T$, then 
$T \leq q(2^{a+1}-1)$, which matches the average lower bound of \Cref{thm:lower_bound} within a factor 4.

If there are $m$ clues on the path from $n_1$ to $n^*$, all at depth exactly $a$ relative to their closest clue ancestor,
then LTS would require $T' = \Omega(2^{m(a-1)})$ steps on average (see \Cref{apdx:lts_clues}), while \rootlts{} takes $T \leq q2^{a+1} = O(q\sqrt[m]{T'})$ steps.
Hence, as long as $q$ is not overly large compared to $m$, the speedup is exponential in $m$.
The best case scenario is when all clue nodes descend from one another, then on average $m=q/2$.
By contrast, the worst case scenario in this setting is when every clue node has $2^a$ clue descendants at relative depth $a$, in which case $q=\Omega(2^{am})$ and then \rootlts{} provides no improvement over LTS --- far too many misleading clues.
(To be continued.)
\end{example}

While we already have an interesting result, a few points remain to be addressed.

\paragraph{Tie breaking.}
It turns out that running BFS with $\Cnm$ is equivalent (up to some tie breaking) to running BFS with $\Cmax$: Indeed, a node with low $\Cnm$-cost can be visited only once its maximum $\Cnm$-cost ancestor has been visited, which forces BFS with $\Cnm$ to behave as if it was BFS with $\Cmax$.
The main difference between $\Cmax$ and $\Cnm$ is that many nodes with different $\Cnm$-costs have the same $\Cmax$-costs.
In essence, the non-monotonicity of $\Cnm$ induces a tie-breaking rule for nodes that all have the same $\Cmax$-costs.
While it can be shown that $\Cnm$ also enjoys the bound of \cref{eq:Cmax_visits},
we can take advantage of the induced tie-breaking to obtain a refined bound.
So we now put $\Cmax$ aside and focus on $\Cnm$, which will require strengthening the analysis to account for non-monotonicity.

\paragraph{Unnormalized weights.}
From now on, the weights $\{\tpr_1, \tpr_2, \dots\}$ need not sum to 1.
Not only is this more convenient, but it also allows for obtaining bounds that are not possible with normalized weights --- see for example \citet{mourtada2017growing}.
In a nutshell, the normalizer of the weights is a linear
factor in the cost function,
and since BFS is invariant to rescaling of the cost function, the normalizer can be omitted.

\paragraph{Off-by-one $\sop$.}\label{par:ob1_sop}
There is one remaining issue with $\Cnm$.
At step $k$, the node $n_k$ is visited,
and a child $n_t$ is pushed into the queue (see \cref{eq:queue_update}).
The actual value of $t$ is not yet known, 
nor is any information that will be gathered between steps $k+1$ and $t$ ---
this includes $\Tprm{< t}$.
But at step $k+1 \leq t$ 
the cost $\Cnm(n_t)$ must be calculated to compare it with the cost of the other nodes in the queue $\queue_{k+1}$.
Since the cost $\Cnm(n_t)$ depends on $\tpr_t$,
the latter cannot depend on the yet unknown quantity $\Tprm{<t}$.
We fix this issue by defining a new set of base cost functions $\{\Crlts_t\}_t$
where $\Crlts_t(n_t) =\infty$ for all $t$, such that $\Crlts(n_t) < \Crlts_t(n_t) / \tpr_t =\infty$ is independent of $\tpr_t$.
In words, only the descendants of $n_t$, but not $n_t$, benefit from the base cost function $\Crlts_t$.
Formally, define the \rootlts{} cost function for BFS:
\begin{align}\label{eq:Crlts}
    \Crlts(n) &= \min_{t\in\Naturals} \frac{\Crlts_t(n)}{\tpr_t}\,,&
    \Crlts_t(n) &= \begin{cases}
    \sopa{n}{n_t}-1 &\text{ if } n_t \RED{\prec} n\,, \\
    \infty&\text{ otherwise.}
    \end{cases}
\end{align}
The $-1$ term plays no significant role for now.
We also set $\Crlts(n_1) =1$, so for all $n\in\desc(n_1)$,
\begin{equation*}
\Crlts(n) = \min_{n_t\RED{\prec} n} \frac{\Crlts_t(n)}{\tpr_t}\,.
\end{equation*}

\begin{algorithm}[thbp]
\caption{
A straightforward pseudo-implementation of the \rootlts{} algorithm.
It is a best-first search using the \emph{non-monotone} cost function $\Crlts$ of \cref{eq:Crlts},
and it satisfies the bound of \Cref{thm:rootlts}.
The calculation of the cost can be sped up, see \Cref{apdx:calc_cost}.
The algorithm can be adapted with a custom stopping criterion,
and additional book-keeping (\eg to retrieve the solution path).
}
\label{alg:rootlts}
\begin{lstlisting}
# $\lstcommentcolor{n_1}$: root node
# $\lstcommentcolor{\tpr}$: rerooter
# $\lstcommentcolor{\pol}$: search policy for Levin Tree Search
def (*\rootlts*)($n_1$, $\tpr$, $\pol$):
  $\queue_1 = \{n_1\}$  # Priority queue
  $t$ = 1
  while $\queue_t \neq \{\}$:

    $\displaystyle n_t = \argmin_{n\in\queue_t} \min_{n_k \prec n} \frac1{\tpr_k}(\sopa{n}{n_k}-1)$
    
    $\queue_{t+1} = (\queue_t \setminus\{n_t\}) \cup \children(n_t)$  # remove $\lcc{n_t}$, insert its children
    $t$ += 1

  return "queue empty"
\end{lstlisting}
\end{algorithm}

The \rootlts{} algorithm (see \Cref{alg:rootlts}) is a best-first search with the (non-monotone, non-self-counting) cost function $\Crlts$.
The max-min form of \cref{eq:Cmax_visits} is a bit difficult to interpret.
The following definition helps to provide a bound that is easier to use in the upcoming main theorem.
\begin{definition}[Subtask decomposition]\label{def:subtask_decomp}
Consider a BFS with the cost function $\Crlts$,
and let $n_T$ be the node visited at step $T$.
A set of $m$ nodes $\{n_{T_i}\}_{i\leq m}$
visited at steps $\{T_i\}_{i\leq m}$
is a \emph{subtask decomposition} of $n_T$
if $n_1 = n_{T_1}\prec n_{T_2}\prec \dots n_{T_m}  = n_T$.
\end{definition}
\Cref{def:subtask_decomp} allows us to think in terms of `subtasks' between selected pairs of large-rerooting-weight nodes (\eg clue nodes) $n_{T_i}$ and $n_{T_{i+1}}$, and to focus on the cost $\Crlts_{T_i}(n_{T_{i+1}})$ of finding $n_{T_{i+1}}$ from $n_{T_i}$ --- rather than having to consider all the nodes on the path from $n_1$ to $n_T$.
We can now state the main theorem, which allows us to choose any convenient subtask decomposition.

\begin{theorem}[\rootlts{} guarantee]\label{thm:rootlts}
Let $n_T$ be the node visited at step $T$ of BFS with the cost function $\Crlts$ of \cref{eq:Crlts}.
Then, for every subtask decomposition $n_{T_1}\prec n_{T_2}\prec \dots n_{T_m}$ of $n_T$,
\begin{align}\label{eq:segment_bound}
    T \leq \sum_{j< m} (\Tprm{<T_{j+1}} - \Tprm{< T_j}) \max_{j \leq i < m}
    \frac1{\tpr_{T_i}}\sopa{n_{T_{i+1}}}{n_{T_i}}\,.
\end{align}
\end{theorem}
That is, the cumulative rerooting weights $(\Tprm{<T_{j+1}} - \Tprm{< T_j})$
of the nodes visited during subtask $j$ are counted toward the cost of the subsequent subtask of maximum cost $\max_{j\leq i < m} \Crlts_i(n_{T_i})/\tpr_{T_i}$.
The proof is in \Cref{apdx:rootlts_proof}.
The core of the proof is similar to the proof \Cref{lem:compose_sccf}, but splitting into segments.
To understand \Cref{thm:rootlts}, it may be easier to start with the following simplification.

\begin{corollary}[Simplified \rootlts{} guarantee]\label{cor:WTmax_bound}
Let $n_T$ be the node visited at step $T$ of BFS with the cost function $\Crlts$ of \cref{eq:Crlts}.
Then, for every subtask decomposition $n_{T_1}\prec n_{T_2}\prec \dots n_{T_m}$ of $n_T$,
\begin{align}
    T &\ \leq\  
    \max_{1 \leq i < m} \frac{\Tprm{<T}}{\tpr_{T_i}}\sopa{n_{T_{i+1}}}{n_{T_i}}\,.
    \label{eq:WTmax_bound}
\end{align}
\end{corollary}
\begin{proof}
The result follows from \Cref{thm:rootlts} by upper bounding
$\max_{j \leq i < m} \dots \leq \max_{1 \leq i < m} \dots$
and telescoping the sum.
\end{proof}
Intuitively, for a chosen subtask decomposition of $n_T$, $T$ is at most
a factor of the $\sop$-cost of the `hardest' weighted subtask.
The rerooting factor $\frac{\Tprm{< T}}{\tpr_{T_i}}$ is the inverse of the normalized weight $\frac{\tpr_{T_i}}{\Tprm{< T}}$.
See \Cref{apdx:compare_lts} for a comparison of \Cref{cor:WTmax_bound} with the LTS bound.

\begin{example}[\Cref{ex:normalized_clue_bound}, continued]\label{ex:qT_clue_bound}
Instead of normalized weights $\tpr_t = \indicator{n_t\in\clueset}/q$, we can now take $\tpr_t = \indicator{n_t\in\clueset}$ without having to know $q$.
Let $q_t$ be the number of clue nodes visited up to step $t$ (included).
Then, if $n^*$ is visited at step $T$, \ie $n^* = n_T$,
we have $\Tprm{\leq T} = q_T$.
We choose the subtask decomposition such that $\{n_{T_i}\}_{i< m}$ are all the $m-1$ clue ancestors of $n_T\!=\!n^*\!=\!n_{T_m}$.
Then, for all $i < m$ we have $\sopa{n_{T_{i+1}}}{n_{T_i}} \leq  2^{a+1}-1$
(see \Cref{ex:normalized_clue_bound}).
Therefore, from \cref{cor:WTmax_bound}, $T < q_T 2^{a+1}$.
Again, this is just a factor 4 away from the lower bound of \Cref{thm:lower_bound}, without needing to know any of the parameters $a, q, \clueset$, and only requiring membership queries $n\mapsto\indicator{n\in\clueset}$.
\end{example}

\begin{example}[\cref{thm:rootlts} vs \cref{cor:WTmax_bound}]\label{ex:WTmax_vs_main_bound}
    See \Cref{fig:easy_ending}.
    There are 6 clue nodes $n_{T_1},\dots n_{T_6}$, all ancestors of $n^*$,
    with $n_{T_1}\!=\!n_1$
    and we set $n_{T_7}\!=\!n_T\!=\!n^*$.
    Assume that $\tpr_t =1$ for $n_t\in\{T_1,\dots T_6\}$, and $\tpr_t=0$ everywhere else.
    The cost for reaching $n_{T_2}$ starting from $n_{T_1}$ is 
    $\Crlts_{T_1}(n_{T_2}) = (\sopa{n_{T_2}}{n_{T_3}}-1) = A$, 
    the cost for reaching $n_{T_4}$ from $n_{T_3}$ 
    is $\Crlts_{T_3}(n_{T_4}) = B$, etc.,  with $A<C<B$.
    \textbf{Steps $T_1=1$ to $T_2$.}
    The BFS starts in $n_1\!=\!n_{T_1}$ and visits nodes with the cost function $\Crlts(\cdot) = \Crlts_{T_1}(\cdot)$, until $n_{T_2}$ is visited at step $T_2$ with cost $\Crlts(n_{T_2})\!=\!A$ --- and at most $A+1$ nodes have been visited.
    \textbf{Steps $T_2$ to $T_3$.}
    At this point, in the priority queue $\queue_{T_2}$ (see \cref{alg:rootlts}),
    every node $n$ has cost $\Crlts_{T_1}(n)\! \geq\! A$,
    but the descendants of $n_{T_2}$ have smaller $\Crlts_{T_2}(\cdot)$ costs.
    Thus, only descendants of $n_{T_2}$ are visited until their $\Crlts_{T_2}(\cdot)$ costs reach $A$, which takes at most $A$ steps.
    That is, $\Crlts_{T_2}$ ``catches up'' with $\Crlts_{T_1}$.
    \textbf{Steps $T_3$ to $T_4$.}
    Then, the clue node $n_{T_3}$ is visited and 
    $\Crlts_{T_3}$ also catches up with $\Crlts_{T_1}$ and $\Crlts_{T_2}$,
    and only descendants of $n_{T_3}$ are visited
    until $\Crlts_{T_3}$ reaches cost $A$.
    But after catching up, no new clue node has been visited, so nodes
    descending from $n_{T_1}, n_{T_2}$ and $n_{T_3}$ are visited
    in equal amounts so that the corresponding costs functions $\Crlts_{T_1}, \Crlts_{T_2}$ and $\Crlts_{T_3}$ maintain equal costs (since their rerooting weights are equal)  until the clue node $n_{T_4}$ is visited, with cost $\Crlts(n_{T_4}) = \Crlts_{T_3}(n_{T_4})=B$.
    \textbf{Steps $T_4$ to $T_5$.}
    At this point, all nodes of cost $\Crlts_{T_i}(\cdot) < B$ for $i=1,2,3$ have been visited, for a total of at most $3B$ nodes, represented by the two dashed triangles rooted at $n_1$ and $n_{T_2}$ and the solid triangle rooted at $n_{T_3}$. 
    Then $\Crlts_{T_4}$ catches up with the previous cost functions, but the clue node $n_{T_5}$ with cost $\Crlts(n_{T_5}) = \Crlts_{T_4}(T_5) = A$
    is visited before catching up entirely.
    \textbf{Steps $T_5$ to $T_6$.}
    Now $\Crlts_{T_5}$ starts catching up with the rest.
    Once $\Crlts_{T_5}$ has caught up with $\Crlts_{T_4}$, both $\Crlts_{T_4}$ (dashed lines) and $\Crlts_{T_5}$ catch up with the rest until $n_{T_6}$ is visited.
    \textbf{Steps $T_6$ to $T$.}
    Then $\Crlts_{T_6}$ starts catching up with the rest until $n^*$ is finally visited.
    The bound of \Cref{thm:rootlts} matches the description given above:
    $T \leq
    \sum_{i < 7}(\Tprm{<T_{i+1}} - \Tprm{<T_i})\max_{i \leq k < 7} \sopa{n_{T_{k+1}}}{n_{T_k}}/\tpr_{T_k}
    = 3B + 2C + A$.
    By contrast, \Cref{cor:WTmax_bound} gives
    $T \leq 6B$ and wrongly suggests that up to $B$ descendants
    specific to each of the 6 clue nodes are visited.
    Note that between the steps $T_4$ and $T$, every node that is visited
    necessarily descends from $n_{T_4}$, since all other nodes
    have costs strictly larger than $A$.
\end{example}

\begin{figure}[t!]
    \centering
    % args: 1=root node, 2=end node, 3=radius at end node, 4=line density
\newcommand{\mytriangle}[4]{
%\draw[fill=#4,color=#4] 
\draw[color=black,#4,line width=1]
let \p1  = #1 in let \p2 = #2 in (\x1, \y1) 
-- (\x2,\y1 - #3cm) -- (\x2, \y1 + #3cm) -- cycle;
}
\begin{tikzpicture}[
        yscale=0.6,
        xscale=1.2,
        mynode/.style={fill=white,draw=black,circle,scale=.75},
    ]
    \begin{scope}[xshift=0cm]
      \node (n1) at (0,0) {};
      \node (n2) at (1,0) {};
      \node (n3) at (2,0) {};
      \node (n4) at (4,0) {};
      \node (n5) at (5,0) {};
      \node (n6) at (6.5, 0) {};
      \node (n7) at (7.5, 0) {};
      \node[left] at (n1) {$n_1$};
      \node[left, inner sep=0] at (n2) {$n_{T_2}$};
      \node[left, inner sep=0] at (n3) {$n_{T_3}$};
      \node[left, inner sep=0] at (n4) {$n_{T_4}$};
      \node[left, inner sep=0] at (n5) {$n_{T_5}$};
      \node[left, inner sep=0] at (n6) {$n_{T_6}$};
      \node (nstar) at (7.5,0.5) {};
      \mytriangle{(n1)}{(n2)}{1}{}
      \mytriangle{(n1)}{(2,0)}{2}{dashed}
      \mytriangle{(n2)}{(n3)}{1}{}
      \mytriangle{(n2)}{(3,0)}{2}{dashed}
      \mytriangle{(n3)}{(n4)}{2}{}
      \mytriangle{(n4)}{(n5)}{1}{}
      \mytriangle{(n4)}{(5.5,0)}{1.5}{dashed}
      \mytriangle{(n5)}{(n6)}{1.5}{}
      \mytriangle{(n6)}{(nstar)}{1}{}
      
      \node[below,xshift=-1.2ex,yshift=-2.4ex] at (n2) {\textcolor{gray}{$A$}};
      \node[below,xshift=-1.2ex,yshift=-2.4ex] at (n3) {\textcolor{gray}{$A$}};
      \node[below,xshift=-1.2ex,yshift=-2.3ex] at (n4) {\textcolor{gray}{$B$}};
      \node[below,xshift=-1.2ex,yshift=-2.4ex] at (n5) {\textcolor{gray}{$A$}};
      \node[below,xshift=-1.2ex,yshift=-2.0ex] at (n6) {\textcolor{gray}{$C$}};
      \node[below,xshift=-1.2ex,yshift=-2.3ex] at (n7) {\textcolor{gray}{$A$}};
      
      \draw let \p1 = (nstar) in  (\x1 - .1cm, \y1) -- (\x1 + .1cm, \y1) node[right] {$n^*$};
      
    \end{scope}
\end{tikzpicture}
    \caption{See \Cref{ex:WTmax_vs_main_bound}.
    The cost of reaching $n_{T_2}$ from $n_1 = n_{T_1}$ using LTS
    (with $\sop$ as the cost function) is 
    $\Crlts_{T_1}(n_{T_2}) = A$.
    The cost of reaching $n_{T_4}$ from $n_{T_3}$ is 
    $\Crlts_{T_3}(n_{T_4}) = B$, etc.
    }
    \label{fig:easy_ending}
\end{figure}

Additionally, from \Cref{cor:WTmax_bound}, we can
choose the subtask decomposition of $n_T$ to be
$n_{T_1}=n_1\prec n_{T_2}=n_T$ (with $m=2$) and obtain the simple bound:
\begin{align}\label{eq:Tw1_bound}
    T &\ \leq\  \frac{\Tprm{<T}}{\tpr_1}\sop(n_T)\,.
\end{align}
Thus, the \rootlts{} bound is always within a factor $\frac{\Tprm{<T}}{\tpr_1}$ of the LTS bound $\sop(n_T)$.

An example using \cref{cor:WTmax_bound} on a Sokoban level can be found in \Cref{apdx:sokoban}.

\section{Robustness to Clue Overload}\label{sec:robustness}

Recall the clue environments of \Cref{sec:lower_bound}.
What if the total number of clue nodes $q$ is very large or infinite?
In this case the bound $T \leq q_T 2^{a+1}$ of \Cref{ex:qT_clue_bound} becomes vacuous, since $q_T$ could be as large as $\Omega(T)$.
Similarly, the bound of \cref{eq:Tw1_bound} compared to LTS alone is also vacuous.
This is \emph{not} an artifact of the proof, and indeed \rootlts{} with uniform weights may never find the solution node $n^*$ in this setup.
To tackle this issue, we first design non-uniform rerooting weights for clue environments, and then generalize the approach to general rerooting weights.

\begin{example}[$\infty$ clues]\label{ex:q_infty}
The policy is uniform.
Define $\tpr_t = \indicator{n_t\in\clueset}/ q_t$ and note that
$\tpr_t = (q_t - q_{t-1}) / q_t = 1- q_{t-1}/q_t \leq \ln (q_t/q_{t-1})$.
Then $\Tprm{\leq T} \leq  1+\sum_{t=2}^T \ln (q_t/q_{t-1}) = 1+\ln(q_T)$.
So,
choosing a subtask decomposition $\{n_{T_i}\}_{i\leq m}$ on clue ancestors of $n^*=n_T$,
from \cref{cor:WTmax_bound}
\begin{equation*}
    T \ \leq\ \Tprm{<T}\,\max_{i < m }\, q_{T_i}\,\sopa{n_{T_{i+1}}}{n_{T_i}}
    \ \leq\  (1+\ln q_T)\,q_{T_{m-1}}\,2^{a+1}\,.
\end{equation*}
Since $q_T \leq T$, this implies
\footnotemark\ that $T= \tilde O(q_{T_{m-1}} 2^a)$.
Since $q_{T_{m-1}}$ does not depend on $T$, $q_{T_{m-1}} 2^a$ is a finite number and $n^*$ is necessarily eventually visited.
Of course, $q_{T_{m-1}}$ might still be a large number, but the lower bound  $\Omega(q2^a)$ in \Cref{thm:lower_bound} shows that this is unavoidable in general --- more interesting bounds may be obtained under specific assumptions.
Moreover, using \cref{eq:Tw1_bound} we also have that $T\leq (1+\ln q_T)q_1\sop(n^*) = \tilde O(2^{d(n^*)})$,
that is, $\rootlts$ with this weighting scheme is also at most within a log factor of LTS.
\comment{For example, if we assume that $q_\tau \leq T^{1/b}$,
then we have $q_\tau = \tilde O(2^{d(n^*)/b})$
and thus combining with the first bound
$T = \tilde O(2^{a+ d(n^*)/b})$ which is now a meaningful bound
compared to $T=O(2^{d(n^*)})$.
}
\end{example}
\footnotetext{
If $T\leq A\ln T$ with $A > 0$,
then also $T \leq A\sqrt{T}$, that is $T \leq A^2$
and thus $T\leq A\ln T\leq 2A\ln A$.
}

\begin{remark}[Off-by-one]
The dependency of $\tpr_t$ on $q_t$ in \Cref{ex:q_infty} would not have been possible using $\Cnm$ of \cref{eq:Cnm} (instead of $\Crlts$ used by \rootlts{}), since $q_t$ is not yet known when $\parent(n_t)$ is visited. See the Off-by-one paragraph on p.\pageref{par:ob1_sop}.
\end{remark}

The following results generalizes the approach to arbitrary \emph{input} rerooting weights $\ttpr$.

\begin{corollary}[Robust \rootlts{} guarantee]\label{cor:lnWTmax}
For all $t\in\Naturals$, let $\ttpr_t$ be the input rerooting weight,
and define
\begin{align*}
    \tpr_t = \frac{\ttpr_t}{\TTprm{\leq t}}\,.
\end{align*}
Let $n_T$ be the node visited at step $T$ of BFS with the cost function $\Crlts$ of \cref{eq:Crlts}.
Then, for every subtask decomposition $n_{T_1}\prec n_{T_2}\prec \dots n_{T_m}$ of $n_T$,
\begin{align}\label{eq:lnWT_bound}
    T &\ \leq\  
    \left(1+\ln\frac{\TTprm{< T}}{\ttpr_1}\right)
    \max_{1 \leq i < m} \frac{\TTprm{\leq T_i}}{\ttpr_{T_i}}\sopa{n_{T_{i+1}}}{n_{T_i}}\,.
\end{align}
\end{corollary}
\begin{proof}
The result follows from \Cref{cor:lnWTmax} by observing that
$\tpr_1 = 1$ and
\begin{align*}
    \tpr_t &= \frac{\ttpr_t}{\TTprm{\leq t}} = 1-\frac{\TTprm{<t}}{\TTprm{\leq t}}
    \leq \ln \frac{\TTprm{\leq t}}{\TTprm{< t}}\,, 
    \quad\text{ and thus }
    &
    \Tprm{<T} &= \tpr_1 + \sum_{t=2}^{T-1} \tpr_t \leq 1+\ln \frac{\TTprm{< T}}{\ttpr_1}\,.
    \qedhere
\end{align*}
\end{proof}

Observe the dependency on $T_i$ instead of $T$ in $\TTprm{\leq T_i}$ --- see \Cref{ex:q_infty} for why this matters.
The weighting of scheme of \Cref{ex:q_infty} can be retrieved by setting  $\ttpr_t = \indicator{n_t\in\clueset}$.

Similarly to \cref{eq:Tw1_bound} we also obtain from \Cref{cor:lnWTmax}
that 
$T \leq (1+\ln(\TTprm{< T}/\ttpr_1)) \sop(n_T)$, that is,
with this weighting \rootlts{} is never worse than a factor $(1+\ln(\TTprm{< T}/\ttpr_1))$ of the LTS bound of $\sop(n_T)$.
Moreover, if $\ttpr_1 = 1$ and $\ttpr_t\leq A$ for all $t$ for some $A$, then $\ln \TTprm{<T} \leq \ln(AT)$ --- or if $\ttpr_1=A$ then $\ln \TTprm{<T}\leq \ln T$. See also a more general transformation of the input rerooting weights in \Cref{apdx:weight_schemes}.

\begin{remark}[Impossible improvement]
One may wonder if the factor $\ln \TTprm{< T}/\ttpr_1$
in \Cref{cor:lnWTmax} can be removed.
Using \Cref{thm:lower_bound}, we show that this is not possible in general.
Define $\ttpr_t\!=\! \indicator{n_t\in\clueset}2^{q_t}$
where $q_t$ is the number of clue nodes visited up $t$.
Then $\TTprm{\leq t} / \ttpr_t\!=\!(2^{q_t+1}-1)/2^{q_t} \leq 2$,
and $\ln \TTprm{<T }/\ttpr_1 \leq q_T\ln 2$.
Thus, for the environments of \Cref{thm:lower_bound},
from \Cref{cor:lnWTmax} we obtain $T = O(q_T 2^a) = O(q 2^a)$,
which is tight since this matches the lower bound of \Cref{thm:lower_bound}.
\end{remark}

\section{Conclusion}

We have proposed a new search algorithm for deterministic domains, called \rootlts{}, which uses a rerooter to start the policy-based LTS at various
nodes in the search tree, potentially speeding up the search significantly.
The rerooter can use side information that the policy does not, such as `clues'.
\rootlts{} comes with theoretical guarantees on the number of search steps (node visits) depending on the quality of the policy and of the rerooter.

Clues appear similar in nature to shaping rewards and landmark heuristics, and 
knowledge from these fields could be leveraged to design domain-specific rerooters and derive guarantees.

While many search domains are deterministic in nature (\eg theorem proving, many games, program synthesis, etc.), it could be valuable to extend \rootlts{} to tackle stochastic domains.

We are eager to see how \rootlts{} performs on real-world applications,
when both the policy and the rerooter are learnt,
in particular since LTS already works well~\citep{orseau2018single,orseau2023ltscm,orseau2021policy}.
We hope that the theory we have proposed can provide a solid foundation for designing faster search algorithms and solving challenging search problems.

\newcommand{\myacks}{
We would like to thank the following people for their help and advice during this project:
Eser Ayg\"un,
Andr\'as Gy\"orgy,
Csaba Szepesvari,
Gellért Weisz,
Tor Lattimore,
Abbas Mehrabian,
Alex Hofer,
Gagan Jain,
Ivo Danihelka,
Ankit Anand,
Alexander Novikov,
Joel Veness,
Matthew Aitchison,
Anian Ruoss,
Grégoire Delétang,
Kevin Li,
and
Tim Genewein.
}
\subsection*{Acknowledgements}\myacks

\bibliographystyle{named}
\bibliography{biblio}

\clearpage
\begin{appendix}
\addtocontents{toc}{\vspace{1em}}

\section{Clues and Classic Algorithms}

Expanding on the claims made in the introduction, we show that classical algorithms (A*, WA*, LTS and MCTS variants) all struggle to make good use of clues.

\subsection{Admissible A* Cannot Use Clues}\label{apdx:admissible}

A* \citep{hart1968formal,dechter1985generalized} is a best-first search algorithm whose priority queue is sorted according to the function $f(n) = g(n) + h(n)$, for each node $n$. The $g$-value of a node is the depth of the path connecting the root of the tree to $n$, while $h(n)$ is an estimate of the number of actions to go.
We denote as $h^*(n)$ the minimum number of actions to connect $n$ to a goal node.
A* is guaranteed to find depth-optimal solutions if $h(n) \leq h^*(n)$ for all $n$. 
Weighted A*~\citep{pohl1970heuristic} is a generalization of A* that uses a weighted version of the heuristic $h$ to ensure $g$-optimality of the return solution within a factor $1+\eps$ for a chosen $\eps \geq 0$.

Regarding the specific clue environment of \Cref{ex:1000clues} of the Introduction, it is actually possible to design an admissible heuristic function for A* that solves the problem as efficiently as \rootlts:
for every node at depth 50 or depth 100 that is not a clue node, set the heuristic to infinity.
This effectively prunes all these nodes.
However, this `rerooting-by-pruning' behaviour is not feasible in general while guaranteeing depth-optimality of the solution.
For example, if there is even a small chance that the solution node 
does not descend from a clue node of depth 50, then no single path can be pruned by any admissible heuristic and A* must search everywhere to guarantee depth-optimality of the returned solution.
The following theorem implies that no algorithm (including A* and Weighted A*) can be claimed to guarantee $1+\epsilon$ $g$-optimality for all monotone cost functions $g$ while also making efficient use of clues.

\begin{theorem}[Bounded suboptimality vs clues]\label{thm:WA*}
Let $a \geq 2$ be a natural number, and let $\eps\geq 0$ be a real number.
There exist sets of environments with $q$ clue nodes where
any algorithm that is guaranteed to return a $1+\eps$ depth-optimal solution
must visit (as per \cref{eq:visited_node}) $\Omega( 2^{qa/(1+\eps)} )$ nodes
while \rootlts{} with rerooting weights $\tpr_t = \indicator{n_t\in\clueset}$
and a uniform policy returns a (non-depth-optimal) solution in $O(q2^a)$ node visits.
\end{theorem}

That is, in some environments, if WA* takes time $T$ to guarantee $1+\eps$ suboptimality, then \rootlts{} can take as little as $O(q\sqrt[q]{T^{1+\eps}})$.

\begin{proof}
Consider the environments of \Cref{thm:lower_bound} (\Cref{sec:lower_bound}),
$\clueset$ is such that all clues descend from one another and are at relative depth exactly $a$ from their closest clue ancestor.
The first solution node $n^*_1$ is placed as a random descendant of the $q$th clue node, thus at depth $qa$.
A second solution node $n^*_2$ is placed randomly at depth $\ceil{qa/(1+\eps)}-1$
under the root --- but does not necessarily descend from another clue node.

Since $d(n^*_1) > (1+\eps)d(n^*_2)$, any algorithm 
that is guaranteed to return a depth-optimal solution within a factor $1+\eps$
must return the solution $n^*_2$.
This requires $2^{\ceil{qa/(1+\eps)}-1}$ node visits on average.
For \rootlts{}, see \Cref{ex:qT_clue_bound} (p. \pageref{ex:qT_clue_bound}).
\end{proof}

\subsection{LTS Struggles with Clues}\label{apdx:lts_clues}

LTS~\citep{orseau2018single} is a best-first search~\citep{pearl1984heuristics}
that uses the cost function $n\mapsto \dop(n) = d(n)/\pol(n)$ to guide the search, where $\pol$ is a policy (see its definition in \Cref{sec:notation}),
and the path probability $\pol(n)$ is the product of the edge (or `action') probabilities
\footnote{No randomness is involved. `Probabilities' should be read as `nonnegative weights that sum to 1.'}
from the root to $n$.
For example, if at each node there are two actions, left and right, with constant respective probabilities $p$ and $1-p$, and the solution node $n^*$ is found after taking $\ell$ lefts and $r$ rights, then $\pol(n^*) = p^\ell(1-p)^r$
and $d(n^*) = \ell + r$. Then, LTS is guaranteed to visit $n^*$ in at most $1+d(n^*)/\pol(n^*) = 1+(\ell + r)/ [p^\ell(1-p)^r]$ node visits~\citep{orseau2018single}.
See also \Cref{apdx:slenderness} for more examples using the new LTS cost function $\sop$ from \Cref{sec:slend}.  

We assume that the policy is proper since this is what existing learning procedures produce, either using cross-entropy~\citep{orseau2018single} or the LTS loss function~\citep{orseau2021policy,orseau2023ltscm}.

Consider the clue environment shown in \Cref{ex:1000clues} of the Introduction.
From the definition of the priority queue \cref{eq:visited_node,eq:queue_update}, 
the path probability of a node $n$ ---
and thus also $\pol(n|\parent(n))$ ---
must be known at the step where the parent $\parent(n)$ is visited,
 so as to calculate the cost of $n$ and compare it with the other nodes in the queue.
Since there is no difference between the children of a node apart from clue information, the min-max optimal choice for the policy is to set $\pol(\ndn|n) = 1/2$ for a child $\ndn$ of $n$ when neither children of $n$ are clue nodes,
and set $\pol(\ndn|n) = 1$ when the child $\ndn$ is a clue node.
Then, assuming the best case that whether a node is a solution can be tested when its parent is visited, the min-max optimal policy sets $\pol(n^*|\parent(n^*)) = 1$.
Therefore, $\pol(n^*) = \prod_{n_1 \prec n\preceq n^*} \pol(n|\parent(n)) = 2^{-98}$.
Hence, $\dop(n^*) = 100\times2^{98}$.
For every node $n$ at depth 98 that does \emph{not} descend from a clue node at depth 50, its cost is $\dop(n) = 98\times 2^{98} < \dop(n^*)$ which means that LTS visits $n$ before $n^*$.
Nodes $n$ at depth 98 descending from a clue node have a cost $98\times 2^{97}$, 
so all nodes of depth 98 are visited before $n^*$ is visited.
Hence the number of nodes that LTS visits before visiting $n^*$ is at least $2^{98+1}-1$, and so $2^{99}$ when including $n^*$.

\subsection{MCTS Struggles}\label{apdx:mcts}

We compare \rootlts{} with MCTS on a couple of \emph{illustrative} examples.
It is important to note that these examples do \emph{not} mean that \rootlts{} is better suited than MCTS for reward-based environments in general (in particular for adversarial or stochastic environments), but they show at least that the linear scaling of the bounds in \Cref{sec:rootlts} and \Cref{sec:robustness} are a reassuring feature of \rootlts.

\subsubsection{D-chain environment}
\citet{coquelin2007bandit,orseau2024superexponential} show that several variants of UCT, including AlphaGo~\citep{silver2016alphago} and its descendants, can take double-exponential time with the depth of the solution node
in the $D$-chain environment --- see \Cref{fig:D-chain}.
How does \rootlts{} behave on this environment?

We take a uniform policy.
We set $\tpr_t$ to the the reward observed at this node, and set $\tpr_1 = 1$.
Like for UCT, all the intermediate rewards/clues are misleading and delay \rootlts{} from finding $n^*$.
However, in this example at least, the exploration/exploitation ratio struck by \rootlts{} is far more balanced.
The analysis is also quite simple.
If $n^*$ is visited at step $T$, then we have $\Tprm{<T }\leq \tpr_1 + \sum_{i< D} (D-1)/D = (D+1)/2$.
The node $n^*$ is at depth $D$, so $\sopa{n^*}{n_1} = 2^{D+1}-1$.
Then from \cref{eq:Tw1_bound} we obtain that
$T \leq (D+1)2^D$, which is only a log factor worse than breadth-first search (which is not misled by the intermediate rewards), and is exponentially faster than the MCTS algorithms mentioned above.

But perhaps these MCTS algorithms are significantly faster than \rootlts{} if the solution node is placed elsewhere in the tree?
Assume that $n^*$ is placed randomly at depth $D$ and is a descendant of
the node $n_2$ (see \cref{fig:D-chain}), which is the very first reward the algorithms may observe.
Since the reward is so high in the tree, MCTS still needs to visit at least $2^D$
descendants of $n_2$.
As in the previous example, \rootlts{} visits at most $(D+1)2^D$ nodes,
which is only a log factor worse.\footnote{
This could be reduced to a log log factor by using the reparameterization of \cref{cor:lnWTmax}.}

What if the solution node $n^*$ is placed at depth $D$ and is a descendant of the node at depth $D/2$ of reward $1/2$?
Let us call this node $n_{\hat t}$, visited at step $\hat t$, and assume $D$ is even for simplicity.
Then the MCTS algorithms may still take double-exponential time
with the depth of $\hat n$ before finding $\hat n$ --- for large enough $D$.
By contrast, in this case \rootlts{} really takes advantage of the rewards to speed up the search:
Using \Cref{cor:WTmax_bound},
if $n^*$ is visited at step $T$,
choosing the subtask decomposition
$\{n_{T_1} = n_1,\, n_{T_2} = n_{\hat t},\, n_{T_3} = n_*\}$
with $m=3$ we obtain
\begin{align*}
    T &\ \leq\  \Tprm{< T}\max\left\{\frac{\Crlts_1(n_{\hat t})}{\tpr_1}, \, \frac{\Crlts_{\hat t}(n^*)}{\tpr_{\hat t}}\right\} 
    \ \leq\  \frac{D+1}2 \max\{2^{D/2+1}-1,\, 2(2^{D/2+1}-1)\} \\
    &\ \leq\  2(D+1)2^{D/2}
\end{align*}
which shows that, within a log factor, \rootlts{} takes the square root of the time that breadth-first search would take to visit $n^*$.

\begin{figure}
    \centering

\def\D{6}

\begin{tikzpicture}[level distance=1cm, sibling distance=2cm]

    \tikzstyle{mynode} = [circle, draw, minimum size=0.5cm]

    \node[mynode] (n0) {$n_1$};
    \foreach \i in {1,...,\D} {
        \pgfmathsetmacro{\j}{int(\i-1)}
        \ifnum \i=2
            \node[above right=0.2cm and 1cm of n\j,minimum size=0.5cm] (n\i) {...};
        \else\ifnum \i=6
            \node[above right=0.2cm and 1cm of n\j,mynode] (n\i) {$n^*$};
        \else
            \node[above right=0.2cm and 1cm of n\j,mynode] (n\i) {};
        \fi\fi
        
        \draw (n\j) -- node[above,midway] {\scriptsize right} (n\i);
        
        \ifnum \i=3
        \else
            \ifnum \i=1
                \node [below right=0.2cm and 1cm of n\j,mynode] (branch\i 2) {$n_2$};
            \else
                \node [below right=0.2cm and 1cm of n\j,mynode] (branch\i 2) {};
            \fi
            \draw (n\j) -- node[above,midway] {\scriptsize left} (branch\i 2);
        \fi
    }
    \node[right=0ex of branch12] () {(D-1)/D};
    \node[right=0ex of branch22] () {(D-2)/D};
    \node[right=0ex of branch42] () {2/D};
    \node[right=0ex of branch52] () {1/D};
    \node[right=0ex of branch62] () {0};
    \node[right=0ex of n6] () {1};

    \foreach \i in {1,...,\D} {
        \pgfmathsetmacro{\j}{int(\i-1)}
        
    }

\end{tikzpicture}

    \caption{The $D$-chain environment. Edge labels are actions,
    and node labels are rewards.
    The binary tree is perfect and infinite.
    UCT, `Polynomial' UCT, AlphaZero and other MCTS variants take double exponential time (and more) with the depth of the node $n^*$ of highest reward.
    }
    \label{fig:D-chain}
\end{figure}
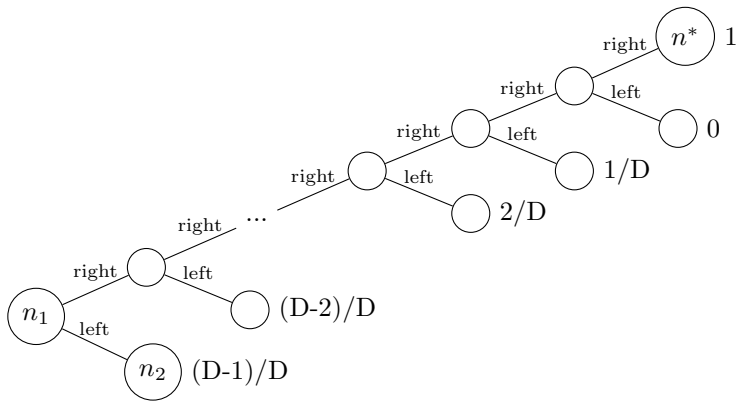

\subsubsection{A measly misleading reward}

\newcommand{\cpuct}{c_{\mathrm{puct}}}

It can even be shown that AlphaZero-like MCTS algorithms can take quadratic time compared to breadth-first search even with just one misleading reward close to the root.
For each child $\ndn$ of a node $n$, define, according to the AlphaZero formula:
\begin{align*}
    B_t(\ndn) = X_t(\ndn) + \cpuct\pol(\ndn|n)\frac{\sqrt{m_t(n)}}{m_t(\ndn) + 1}
\end{align*}
where $X_t(\ndn)$ is taken to be the average reward observed on the descendants of $\ndn$ (included),
and $m_t(n)$ is the number of times the node $n$ has been traversed before step $t$.
We assume that $\cpuct=2$.
At every expansion step, the tree is traversed from the root $n_1$
and, at every parent node, its child with maximum $B$ value is selected for traversal.

Consider an infinite perfect binary tree where
the two children $n_2$ and $n_3$ of the root are such that there is a reward
$\alpha\in(0, 1)$ at $n_2$.
We assume that $n_2$ is visited before $n_3$.
For both \rootlts{} and AlphaZero we take a uniform policy.
Then $\cpuct \pol(\ndn|n) = 1$ for all $\ndn\in\children(n)$.
The node $n^*$ has a reward of 1 and descends from $n_3$.
If $n^*$ is visited at step $T$,
we must have $B_T(n_3) \geq B_T(n_2)$ which implies $B_T(n_3) \geq \alpha$.
To visit $n^*$, the MCTS algorithm must first visit all descendants of $n_3$ at depth less than $n^*$, so it must traverse $n_3$ at least $2^{d(n^*)-1}$ times,
and thus we must have $m_T(n_3) \geq 2^{d(n^*)-1}$.
Since $X_T(n_3) = 0$ for all $t$, it follows that 
\begin{align*}
    \frac{\sqrt{m_T(n_1)}}{m_T(n_3) + 1} \geq \alpha
\end{align*}
and since $m_T(n_3) \geq 2^{d(n^*)-1}$ then $T = m_T(n_1) \geq (\alpha 2^{d(n^*)-1})^2$.

By contrast, breadth-first search takes at most $2^{d(n^*)+1}$ steps to visit $n^*$.
For \rootlts, from \cref{eq:Tw1_bound}, assuming $\tpr_1 = 1$,
an additional clue of weight $\alpha$ only increases the leading factor $\Tprm{<T}$ of the bound by $\alpha$ and it visits $n^*$ in at most $(1+\alpha)2^{d(n^*)+1}$ node visits.
That is, if breadth-first search takes $T^*$ steps to find $n^*$, then
AlphaZero takes $\Omega((\alpha T^*)^2)$ steps while \rootlts{} takes
$O((1+\alpha)T^*)$.

\section{Slenderness}\label{apdx:slenderness}

This appendix provides further details and examples regarding the slenderness cost function introduced in \Cref{sec:slend},
as well as the proof of \cref{eq:sop_bounds}.

\subsection{More Slenderness Examples}

\begin{example}[$p$ left, $1-p$ right]\label{ex:slend_left_right}
Consider an infinite binary tree where the conditional probability of the left child is always $p$ and that of the right child is always $1-p$.
Then, starting at the root and taking $\ell$ times the left child and $r$ times the right child, in any order, the corresponding node $n$ satisfies
\begin{align*}
    \sop(n) \leq \max\left\{\frac1p,\frac1{1-p}\right\}\frac1{\pol(n)}\,,
\end{align*}
with $\pol(n) = p^\ell(1-p)^r$.
Compare with 
$\dop(n) = (\ell+r)/\pol(n)$
--- and recall that $\sop(n) \leq 1+\dop(n)$.
Also note that this means that the slenderness $\slend(n) \leq \max\{1/p,1/(1-p)\}$.
\end{example}

\begin{proof}[\Proofof{} \Cref{ex:slend_left_right}]
First, from \cref{eq:sop_anc} we show that for a given unordered sequence of conditional probabilities,
the node with the largest $\sop$-cost is the one for which the smallest conditional probabilities are nearest to the root. Indeed, two nodes $n_b$ and $n_c$ that both descend from a node $n_a$ with sequences of conditional probabilities $(a, b)$ and $(b,a)$ respectively.
Then, using \cref{eq:sop},
\begin{align*}
    \sop(n_b) &= \sop(n_a) + \frac1{a\pol(n_a)} + \frac1{ab\pol(n_a)}\,, \\
    \sop(n_c) &= \sop(n_a) + \frac1{b\pol(n_a)} + \frac1{ab\pol(n_a)}
\end{align*}
and thus $\sop(n_b) \geq \sop(n_c)$ if and only if $a \leq b$.

Now, back to the example, let us assume first that $p > 1-p$.
The ordering of the $\ell$ lefts and $r$ rights that maximizes $\sop(n)$ is if all the rights (with probabilities $1-p$) come first.
Hence,
\begin{align*}
    \sop(n) &\leq 1+\frac{1}{1-p}+\dots \frac1{(1-p)^r} ~+~ \frac1{(1-p)^rp} + \dots
    \frac1{(1-p)^rp^\ell} \\
    &= 1+\frac{1}{1-p}+\dots \frac1{(1-p)^{r-1}} + \frac1{(1-p)^r}\left(1+\frac1p+\dots\frac1{p^\ell}\right) \\
    &= \frac{\frac1{(1-p)^r}-1}{\frac1{1-p} - 1} + \frac1{(1-p)^r}\frac{\frac1{p^{\ell+1}}-1}{\frac1p -1} \\
    &= \frac1{(1-p)^r} \left(\frac{1-p}{p} + \frac1{1-p}\frac1{p^\ell} - \frac{p}{1-p}\right) - \frac{1-p}{p} \\
    &\leq \frac1{1-p}\frac1{(1-p)^rp^\ell}\,.
\end{align*}
The case $p \leq 1-p$ is proven similarly.
\end{proof}

With a bit more work,
using $1+1/p+\dots1/p^\ell \leq 1+\ell/p^\ell$, the bound can be improved to
\begin{align*}
    \sop(n) \leq \max\left\{\min\left\{r+2, \frac1{p_{\mathrm{left}}}\right\},~
    \min\left\{\ell+2, \frac1{p_{\mathrm{right}}}\right\}
        \right\}
    \frac1{\pol(n)}\,,
\end{align*}
where $p_{\mathrm{left}} = p = 1- p_{\mathrm{right}}$.
\comment{Actually, we can get $\min\{r + 1/(1-p), 1/p\},\dots$
Even more generally,
we can get
\begin{align}
    \slend(n) \leq \min_{0\leq i \leq d(n)}\left\{ i + \frac1{1-p_{i+1}}\right\}
\end{align}
with $p_{d(n)+1}=0$,
where $p_i$ is the $i$th \emph{largest} conditional probability on the path to $n$.
Note that if $p_\ell> p_r$, and there are $\ell$ lefts,
then the first $p_r$ is for $i=\ell+1$.
Note that this is not always tight, for example with Laplace estimator,
on the leftmost path we have $1+\ln d(n)$ but the bound above would give $\Omega(d(n))$. It's because the above display considers the worst case
ordering of the conditionals, but in the best case it can be much lower.
}

\comment{One may wonder if we could obtain a bound of the form $\sop(n) = O(d(n) + 1/\pol(n))$
but taking
$\pol_1 = (1/2)^\alpha, \pol_2 = (2/3)^\alpha, \pol_3 = (3/4)^\alpha\dots$
leads to $1/\pol_{1:d(n)} = (d(n)+1)^\alpha$,
and $\sop(n) = 1^\alpha+2^\alpha+3^\alpha+4^\alpha+\dots d(n)^\alpha 
\geq d(n)^{1+\alpha} / (1+\alpha) = 
\Omega(d(n)/\pol(n) / (1+\alpha))$.
}

\begin{example}[Depth-dependent arity]\label{ex:depth_arity_sop}
Consider an infinite tree where the number of children of any node
depends exclusively on the depth of the node, in an arbitrary way.
Let $C_d$ be the number of children of any node $n$ of depth $d(n) = d$.
Assume that the policy is uniform.
Then for any node $n$, $\sop(n)$ is exactly the number of nodes of depth at most $d(n)$ (and also exactly the number of nodes of cost at most $\sop(n)$).
The proof is as follows.
If the number of nodes of depth exactly $d$ is $L$ and the number of nodes of depth at most $d$ is $N$,
then the number of nodes of depth exactly $d+1$ is $LC_d$ and the number of nodes of depth at most $d+1$ is $N+LC_d$.
Assume that $\sop(n) = N$
and that $1/\pol(n) = L$,
then for all children $\ndn\in\children(n)$,
$1/\pol(\ndn) = 1/(\pol(n)\pol(\ndn| n)) = LC_d$
and $\sop(\ndn) = \sop(n) + 1/\pol(\ndn) = N+LC_d$.
Since $\pol(n_1) = \sop(n_1) = 1$, the result follows by induction.
\end{example}

\subsection{Proof of the Slenderness Bounds}

Refer to \Cref{ex:depth_arity_sop} for some intuition about why $\sop$ `counts' the number of nodes in a tree.

Before proving the bounds of \cref{eq:sop_bounds},
we build an intermediate concept we call the \emph{complementary policy} with respect to a tree,
and we prove that it satisfies some useful properties.

Recall that a set $\nodeset'\subseteq\nodeset$ of nodes is a tree
rooted in some node $n_a$
if for every node in the set, its parent is also in the set, except for the root $n_a$: $\forall n\in\nodeset'\setminus\{n_a\}, \parent(n)\in\nodeset'$.

\begin{definition}[Complementary policy]\label{def:comp_pol}
For a given tree $\nodeset'\subseteq\nodeset$ rooted in $n_1$, the $\nodeset'$-complementary policy $\pol'$ of $\pol$
is, for all $n\in\nodeset$,
\begin{align*}
    \pol'(n)\ =\ \pol(n) - \sum_{\ndn\in\children(n)\cap \nodeset'} \pol(\ndn)
    \ =\ \pol(n)\left(1- \sum_{\ndn\in\children(n)\cap \nodeset'} \pol(\ndn|n)\right)\,.
    &\qedhere
\end{align*}
\end{definition}
The quantity $\pol'(n)$ is the amount of the path probability $\pol(n)$ 
that is \emph{not} passed down to children of $n$ \emph{within} $\nodeset'$.
That is, $\pol'(n)$ is `lost' (or `leaked') by $n$ from $\nodeset'$.
This also accounts for the fact that the policy may not be proper
--- which is a form of `leakage' too.
Note that a leaf of $\nodeset'$ loses all its path probability $\pol(n)$, that is, $\pol'(n) = \pol(n)$.
Then we show that the total path probability that is lost is exactly 1 (the path probability of the root).

\begin{lemma}\label{lem:comp_pol_sum1}
For any tree $\nodeset'$ rooted in $n_1$, the $\nodeset'$-complementary policy $\pol'$ of $\pol$ satisfies
\begin{align*}
    \sum_{n\in\nodeset'} \pol'(n) = 1\,.
    &\qedhere
\end{align*}
\end{lemma}
\begin{proofnoqed}
Using \Cref{def:comp_pol}:
\begin{align*}
\sum_{n\in\nodeset'} \pol'(n) &= 
\sum_{n\in\nodeset'} \pol(n) - 
\sum_{n\in\nodeset'}\ {}
\sum_{\ndn\in\children(n)\cap\nodeset'} \pol(\ndn)  \\
&= 
\sum_{n\in\nodeset'} \pol(n) - 
\sum_{n\in\nodeset'\setminus\{n_1\}} \pol(n) 
\ = \ \pol(n_1)\  =\  1\,.
\qedhere
\end{align*}
\end{proofnoqed}

In the next result,
the quantity $\frac{\pol'(n)}{\pol(n)}\slend(n)$ corresponds to the \emph{fraction} of the share $\slend(n)$
that is not passed down to its children within $\nodeset'$, that is,
the fraction that is `lost'.
The next result shows that every share is eventually lost, either partially by some nodes of $\nodeset'$,
or in full by the leaves of $\nodeset'$ (since they pass down no share at all within $\nodeset'$).
Hence, 
since the share $\slend(n)$ represents a fraction of the number of ancestors of $n$ (included),
when summing over all nodes of $\nodeset'$, the cumulative fraction of the lost shares
is equal to the number of nodes in $\nodeset'$.

\begin{lemma}\label{lem:comp_pol_sum_slend}
For any tree $\nodeset'$ rooted in $n_1$, the $\nodeset'$-complementary policy $\pol'$ of $\pol$ satisfies
\begin{align*}
    |\nodeset'| = \sum_{n\in\nodeset'} \frac{\pol'(n)}{\pol(n)}\slend(n)\,.
    &\qedhere
\end{align*}
\end{lemma}
\begin{proofnoqed}
We have
\begin{subequations}
\begin{align}
    \sum_{n\in\nodeset'} \frac{\pol'(n)}{\pol(n)}\slend(n)
    &= \sum_{n\in\nodeset'} \pol(n)\left(1-\sum_{\ndn\in\children(n)\cap\nodeset'}\pol(\ndn|n)\right)\frac{\slend(n)}{\pol(n)} \\
    &= \sum_{n\in\nodeset'} \slend(n)
    - \sum_{n\in\nodeset'} \sum_{\ndn\in\children(n)\cap\nodeset'}\pol(\ndn|n) \slend(n) \\
    &= \sum_{n\in\nodeset'} \slend(n)
    - \sum_{n\in\nodeset'} \sum_{\ndn\in\children(n)\cap\nodeset'}(\slend(\ndn) -1) \\
    &= \slend(n_1) + \sum_{n\in\nodeset'\setminus\{n_1\}} 1 \\
    &= |\nodeset'|\,,
\end{align}
\end{subequations}
with 
\begin{eqitems}
\item from the definition of $\pol'$ in \Cref{def:comp_pol},
\item from the definition of $\slend$ in \cref{eq:slend},
\item by cancellation of every $\slend(n)$ except for the root $n_1$,
and dealing with the $-1$ terms separately,
\item since $\slend(n_1)=1$.
\qedhere
\end{eqitems}
\end{proofnoqed}

We can now easily prove the upper bound of \cref{eq:sop_bounds}.
For any given $\theta \geq 0$, define
\begin{align*}
    \nodesettheta = \left\{n:\sop(n) \leq \theta\right\}
\end{align*}
to be the set of nodes of $\sop$-cost at most $\theta$,
and observe that it forms a tree rooted in $n_1$, due to the monotonicity of $\sop$.

\begin{proposition}[$\sop$ SCCF]\label{prop:sop_selfcounting}
The slenderness cost function $\sop$ is self-counting.
\end{proposition}
\begin{proofnoqed}
The following holds for all $\theta \geq 0$.
Let $\pol_\theta$ be the $\nodesettheta$-complementary policy of $\pol$.
Then from \Cref{lem:comp_pol_sum_slend}, the definition of $\nodesettheta$,
followed by \Cref{lem:comp_pol_sum1},
\begin{align}
    |\nodesettheta| =
    \sum_{n\in\nodesettheta} \frac{\pol_\theta(n)}{\pol(n)}\slend(n) 
    = \sum_{n\in\nodesettheta} \pol_\theta(n)\sop(n)
    \leq \sum_{n\in\nodesettheta} \pol_\theta(n)\theta = \theta\,.
    &\qedhere
\end{align}
\end{proofnoqed}

\newcommand{\leafsetthetaplus}{\leafset_\theta^+}

Now, for the lower bound of \cref{eq:sop_bounds}, define the `child-expansion' of $\nodesettheta$:
\begin{equation*}
    \nodesetthetaplus\ =\ \nodesettheta\cup \bigcup_{n\in\nodesettheta}\children(n) \,,
\end{equation*}
and let $\leafsetthetaplus = \nodesetthetaplus \setminus \nodesettheta$ be the set of leaves of $\nodesetthetaplus$.
Recall that for a proper policy $\pol$, for every node $n$, $\sum_{\ndn\in\children(n)} \pol(\ndn| n) = 1$.

\begin{lemma}[$\theta$ bounds]\label{lem:theta_bounds}
For all $\theta\geq 0$,
\begin{equation*}
    |\nodesettheta|\ \leq\ \theta\ <\ |\nodesetthetaplus|\,,
\end{equation*}
where the rightmost inequality holds only if the policy is proper.
\end{lemma}
\begin{proof}
The left-hand-side follows from \Cref{prop:sop_selfcounting}.
For the right-hand-side,
let $\pol_\theta^+$ be the $\nodesetthetaplus$-complementary policy of $\pol$.
For every $n$ of $\nodesettheta$, all its children are in $\nodesetthetaplus$
and thus, since the policy is proper, $\pol_\theta^+(n) = 0$ (by \Cref{def:comp_pol}).
Moreover, by \Cref{lem:comp_pol_sum1},
\begin{equation*}
    \sum_{n\in\leafset_\theta^+} \pol_\theta^+(n) = 1\,.
\end{equation*}
Therefore, by \Cref{lem:comp_pol_sum_slend},
\begin{equation*}
    |\nodesetthetaplus| = \sum_{n\in\nodesetthetaplus}\frac{\pol_\theta^+(n)}{\pol(n)}\slend(n)
    = \sum_{n\in\leafsetthetaplus}\frac{\pol_\theta^+(n)}{\pol(n)}\slend(n)
    = \sum_{n\in\leafsetthetaplus}\pol_\theta^+(n) \sop(n) 
    > \sum_{n\in\leafsetthetaplus}\pol_\theta^+(n) \theta = \theta\,.
\end{equation*}
where the inequality follows by $\sop(n) > \theta$ for $n\in\leafsetthetaplus$
since $n\notin\nodesettheta$.
\end{proof}

\begin{proposition}[$\sop$ lower bound]\label{prop:sop_lower_bound}
Let $\theta \geq 0$.
Let $B$ be the average branching factor of $\nodesettheta$:
\begin{equation*}
    B= \frac1{|\nodesettheta|}\sum_{n\in\nodesettheta}|\children(n)|\,.
\end{equation*}
If the policy $\pol$ is proper then 
\begin{align*}
    |\nodesettheta| > \frac{\theta -1}{B}\,.
    &\qedhere
\end{align*}
\end{proposition}
\begin{proof}
The result follows from \Cref{lem:theta_bounds}:
\begin{equation*}
    \theta < |\nodesetthetaplus| = 1+\sum_{n\in\nodesettheta}|\children(n)|
    = 1+|\nodesettheta|B\,,
\end{equation*}
and rearranging.
\end{proof}

\subsection{Telescoping Property}

We show that the root-dependent cost $\sop$ function 
satisfies a form of `telescoping' property.
\begin{lemma}[$\sop$ telescopes]\label{lem:sop_telescopes}
For all  $n_a, n_b, n\in\nodeset$
such that $n_a \preceq n_b \preceq n$,
\begin{equation*}
    \frac{\sopa{n}{n_a}-1}{\pol(n_a)} = \frac{\sopa{n_b}{n_a}-1}{\pol(n_a)} + \frac{\sopa{n}{n_b}-1}{\pol(n_b)}\,.
\end{equation*}
\end{lemma}
This is of the form $f_a(n) = f_a(n_b) + f_b(n)$.
\begin{proof}
From \cref{eq:sop_anc_root},
\begin{align*}
    \sopa{n}{n_a} &= \sum_{n_a \preceq \nup \preceq n} \frac{1}{\pol(\nup|n_a)} \\
    &= \sum_{n_a\preceq \nup\preceq n_b} \frac{1}{\pol(\nup|n_a)}
    + \sum_{n_b\prec \nup \preceq n} \frac{1}{\pol(\nup|n_a)} \\
    &= \sopa{n_b}{n_a}
    + \frac1{\pol(n_b|n_a)}\sum_{n_b \prec \nup \preceq n} \frac{1}{\pol(\nup|n_b)} \\
    &= \sopa{n_b}{n_a} + \frac{1}{\pol(n_b|n_a)}\left(\sopa{n}{n_b}-1\right)\,,
\end{align*}
and the result follows by subtracting 1 on each side, then dividing by $\pol(n_a)$.
\end{proof}

Note that, by contrast, the rooted $\dopa{n}{n_a} = (d(n) - d(n_a)) / \pol(n|n_a)$ does not have a similar telescoping form (with the offending term in red):
\begin{align*}
    \frac{d(n)-d(n_a)}{\pol(n|n_a)} &= \frac1{\pol(n_b|n_a)}\frac{d(n)- d(n_b)}{\pol(n|n_b)} + \frac{d(n_b)-d(n_a)}{\RED{\pol(n|n_b)}\pol(n_b|n_a)}\,, \\
    \dopa{n}{n_a} &= \frac1{\pol(n_b|n_a)}\dopa{n}{n_b} + \RED{\frac1{\pol(n|n_b)}}\dopa{n_b}{n_a}\,, \\
    \frac{\dopa{n}{n_a}}{\pol(n_a)} &= \frac{\dopa{n}{n_b}}{\pol(n_b)} + \RED{\frac1{\pol(n|n_b)}}\frac{\dopa{n_b}{n_a}}{\pol(n_a)}\,.
\end{align*}

\section{Self-Counting Cost Functions}\label{apdx:self_counting}

\begin{example}[Simple monotone self-counting cost functions]
In a binary tree, the cost function $c(n) = 2^{d(n)+1}-1$ is monotone self-counting --- it is also quite tight for perfect infinite binary trees.
The cost function $c(n) = 3^{d(n)+1}$ is also monotone self-counting, even though the it grossly overestimates the number of nodes of cost at most a given bound: for a given node $n^*$ there are far fewer than $3^{d(n^*)+1}$ nodes of cost at most $3^{d(n^*)+1}$ --- since there are only $2^{d(n^*)+1}-1$ of them.
In a binary tree, the cost function $c(n) = 2000$ for the 1023 nodes of depth at most 9 and $c(n) = \infty$ otherwise is also monotone self-counting.
In a chain (every node has only one child), the cost function $c(n) = d(n)+1$ is also monotone self-counting.
\end{example}

\begin{example}[Simple non-monotone self-counting cost-function]
In a binary tree, the cost function $c(n) = 1024$ for all the 1024 nodes at depth 10, and $c(n)=\infty$ for all the other nodes is self-counting but is not monotone. Hence \cref{lem:t_leq_cnt} does not apply.
Indeed, the root has infinite cost, but it is still the first node enumerated by BFS.
\end{example}

\begin{remark}[Inverse probability SCCF]\label{rmk:proba_dist_self-counting}
The inverse of a probability distribution over the nodes is a self-counting cost function.
Indeed, for all $\theta \geq 0$,
\begin{equation*}
    \left|\left\{n: \frac{1}{p(n)} \leq \theta \right\}\right|
    = |\{n: 1 \leq p(n)\theta \}|
    = \sum_{n: 1 \leq p(n)\theta} 1
    \leq \sum_{n: 1 \leq p(n)\theta} p(n)\theta \leq \theta\,.
\end{equation*}
However, a self-counting cost function is not necessarily the inverse of a probability distribution over the nodes.
\end{remark}

\begin{remark}[Harmonic composition of SCCF]
It is tempting to attempt to use the harmonic mean instead of the minimum so as to accumulate the weights for nodes that are shared.
Unfortunately, here is a simple counter-example:
The nodes $n_a$ and $n_b$ are children of the root $n_1$.
The costs are $c_1 = (1, 2, 3)$ for nodes $n_1, n_a, n_b$, and $c_2=(1, 3 , 2)$.
We take $\tpr_1 = \tpr_2 = 1/2$.
Then $c(n_a) = c(n_b) = 1/(1/2\cdot1/2 + 1/2\cdot1/3) = 12/5 < 3$, 
and thus $|\{n: c(n) \leq 12/5 \}| \not\leq 12/5$, so $c$ is not self-counting.
However, if $c_1$ and $c_2$ are inverse probability distributions (and thus are self-counting, see \Cref{rmk:proba_dist_self-counting}),
then their composition using a harmonic mean is
also the inverse of a probability distribution and thus is also self-counting.
\end{remark}

\section{Proof of \texorpdfstring{\Cref{thm:rootlts}}{the Main Theorem}}\label{apdx:rootlts_proof}

First we build a couple of tools to deal with non-monotone cost functions,
and compositions of non-monotone self-counting cost functions.
Then we prove the main theorem.

The following lemma holds for any cost function $c$,
and relates the costs of the nodes anywhere in the search tree to the costs of the ancestors of a particular path of interest.
It says that, if $n_t$ is a descendant of $n_j$, then
all the nodes visited by the best-first search between steps $j$ (excluded) and $t$ have cost at most the maximum cost of the nodes on the path from $n_j$ (excluded) to $n_t$.

\makenamedref{treetopath}{Tree-to-path}

\begin{lemma}[\treetopathname]\treetopathlabel
Let $n_t$ be the node visited by BFS with a (possibly non-monotone) cost function $c$ for all $t\geq 1$.
Then, for all nodes $n_j$ and $n_t$ such that $n_j \prec n_t$, we have 
\begin{align*}
    \max_{j < k \leq t} c(n_k) ~=~ \max_{n_j \prec n \preceq n_t} c(n)\,.
    &\qedhere
\end{align*}
\end{lemma}
\begin{proof}
Define $\theta = \max_{n_j \prec n \preceq n_t} c(n)$.
Since $n_j\in Q_j$ (by \cref{eq:visited_node}),
by \cref{eq:queue_update} the priority queues
$\queue_{j+1}, \dots \queue_{t}$ each contain one of the nodes
on the path from $n_j$ to $n_t$
(\ie a node $n$ such that $n_j \prec n \preceq n_t$)
of cost at most $\theta$.
Hence, at step $k$ with $j<k\leq t$
any node in the priority queue $Q_k$ of cost strictly more that $\theta$
cannot be chosen by \cref{eq:visited_node} to be $n_k$.
Therefore, for all $j <k\leq t$, $c(n_k) \leq \theta$,
that is,
$\max_{j < k \leq t} c(n_k) \leq \theta = \max_{n_j \prec n \preceq n_t} c(n)$.
Equality follows due to set inclusion.
\end{proof}

\begin{remark}
Excluding $n_j$ in \treetopathref{} is more general than including it in the maximum:
Let $n_a$ be a node with cost 100,
$n_b$ a child of $n_a$ with cost 10,
$n_c$ a descendant of $n_b$
and assume that the maximum cost between $b$ and $c$ is 10.
Then \treetopathref{} says that 
not only all the nodes visited between  steps $b$ and $c$ (both included) have cost at most 10,
but also that all nodes visited between steps $a$ (excluded) and $b$ have cost at most 10.
\end{remark}

Next, we prove a versatile bound for compositions of (possibly non-monotone) self-counting cost functions.
It generalizes \cref{eq:compose_t_leq_cnt} which applies only to composition of monotone self-counting cost functions.

\makenamedref{generalsccfcomposition}{SCCF composition bound}

\begin{lemma}[\generalsccfcompositionname]\generalsccfcompositionlabel
Let $c(n) = \min_{i \in [N]} \frac{c_i(n)}{\tpr_i}$
for all $n\in\nodeset$
be a composition of $N$ base self-counting cost functions (possibly non-monotone).
Let $n_T$ be the node visited by BFS at step $T$ with the cost function $c$.
Let $S_{i, T}$ be the set of nodes where the weighted base cost function $c_i / \tpr_i$ is the minimum cost in $c$:
\begin{equation*}
    S_{i, T} = \left\{n_t: t\leq T,\ \frac{c_i(n_t)}{\tpr_i} = c(n_t) \right\}\,.
\end{equation*}
Then 
\begin{align*}
    T\ \leq\ \sum_{i\in[N]} \tpr_i \max_{n \in S_{i,T}} c(n)\,.
    &\qedhere
\end{align*}
\end{lemma}

Note that \Cref{eq:Cmax_visits} can be retrieved immediately by relaxing $n\in S_{i,t}$ to $n_t: 1 < t \leq T$ and using \treetopathref.

Also, by definition of $S_{i, T}$,
\Cref{generalsccfcomposition_label} implies $T\ \leq\ \sum_{i\in[N]} \max_{n \in S_{i,T}} c_i(n)$
and, recalling that the cost of a self-counting cost function is an upper bound on the number of steps that BFS takes,
this bound can be interpreted as follows:
The total number of steps of BFS with the compound cost function $c$ is the sum of (an upper bound on) the number of steps of BFS with each base cost function.

\begin{proofnoqed}
We have
\begin{subequations}
\begin{align}
    T = |\{n_t: t\leq T\}|
    &=\left|\left\{n_t: t\leq T,\ \RED{\min_{i\in [N]}\frac{c_i(n_t)}{\tpr_i} = c(n_t)} \right\}\right|  \\
    &=\left|\bigcup_{i\in[N]}\left\{n_t: t\leq T,\ \RED{\frac{c_i(n_t)}{\tpr_i} = c(n_t)} \right\}\right| = \left|\bigcup_{i\in [N]} S_{i, T}\right| \\
    &=\left|\bigcup_{i\in [N]}\left\{n\in S_{i, T}: 
    \frac{c_i(n_t)}{\tpr_i} 
    \leq \max_{\dot n\in S_{i, T}} c(\dot n) \right\}\right|  \\
    &\leq\left|\bigcup_{i\in [N]}\left\{\RED{n}: c_i(\RED{n})
    \leq \RED{\tpr_i}\max_{\dot n\in S_{i, T}} c(\dot n) \right\}\right|  \\
    &\leq \RED{\sum_{i\in [N]}}\left|\left\{n: c_i(n) 
    \leq \tpr_i\max_{\dot n\in S_{i, T}} c(\dot n) \right\}\right|  \\
    &\leq \sum_{i\in [N]} \tpr_i \max_{\dot n\in S_{i, T}} c(\dot n)
\end{align}
\end{subequations}
with
\begin{eqitems}
\item by adding a redundant condition and by definition of $c$,
\item by spreading the $n_t$ into different sets (possibly with repetition),
\item by (redundantly) upper bounding the costs $c(n_t)$ with the maximum cost of the elements of the same sets,
\item by dropping the condition on $n$ and moving $\tpr_i$,
\item using a union bound,
\item using that every $c_i$ is self-counting.
\qedhere
\end{eqitems}
\end{proofnoqed}

Now we turn our attention to $\Crlts$, which is a composition of $\{\Crlts_t\}_t$ --- see \cref{eq:Crlts}.
For all $t$, the cost function $\Crlts_t(\cdot)$ is monotone on the descendants of $n_t$ (excluded).
First, we show that, like for rooted $\sop$, all the base cost functions $\{\Crlts_t\}_t$ are also self-counting.
\begin{lemma}[\boldmath $\Crlts_t$ SCCF]
For all $t\geq 1$,
if $n_t$ is the node visited by BFS with the cost function $\Crlts$,
then the cost function $\Crlts_t(\cdot)$ is self-counting.
\end{lemma}
\begin{proof}
For all $\theta \geq 0$, for all $t\geq 1$,
recalling that $\Crlts_t(n) =\infty$ for $n_t \not\prec n$, we have
\begin{align*}
    |\{n\in\nodeset:\Crlts_t(n) \leq \theta\}|
    &= 
    |\{n\in\desc(n_t): \Crlts_t(n) \leq \theta\}| \\
    &= 
    |\{n\in\desc(n_t): \RED{\sopa{n}{n_t}} \leq \theta\RED{{}+{}1}\}| \\
    &= 
    |\{n\in\RED{\descn}(n_t): \sopa{n}{n_t} \leq \theta+1\}|\RED{{}-{}1}
    \leq \theta
\end{align*}
where we used that $\sopa{\cdot}{n_t}$ is self-counting on the last inequality.
\end{proof}

We can now specialize \generalsccfcompositionref{} to $\Crlts$.
\begin{corollary}[\boldmath $\Crlts$ composition bound]\label{lem:Crlts_composition}
Let $n_T$ be the node visited by BFS at step $T$ with the cost function $\Crlts$.
Then 
\begin{align*}
    T\ \leq\ 1+\sum_{i < T} \tpr_i \max_{i < t \leq T} \Crlts(n_t)\,.
    &\qedhere
\end{align*}
\end{corollary}
\begin{proof}
The specific cost $\Crlts(n_1) = 1$ at the root (see \cref{eq:Crlts})
is equivalent to including in $\Crlts$ a trivial self-counting cost function $\Crlts_0$ such that $\Crlts_0(n_1) =1$ and $\Crlts_0(n) =\infty$ for all $n\neq n_1$,
and with $\tpr_0 = 1$.
Then observe that $\max_{n\in S_{0, T}}\Crlts(n) = 1$.
The result then follows from the definition of $\Crlts$ in \cref{eq:Crlts}
and \generalsccfcompositionref,
and relaxing $n\in S_{i,T}$ to $n_t$ being visited strictly after the root $n_i$ of the base cost function $\Crlts_i$ is visited.
\end{proof}

At this point, from \Cref{lem:Crlts_composition} we can easily obtain an equivalent of \Cref{cor:WTmax_bound} by relaxing $\max_{i < t \leq T}\Crlts(n_t) $ to $\max_{\RED{1}<t\leq T}\Crlts(n_t)$ and then using \treetopathref{} to get:
\begin{equation*}
    T \leq 1+\Tprm{< T} \max_{n_1 \prec n\preceq n_T} \Crlts(n)\,.
\end{equation*}

Now, to prove the more general bound of \Cref{thm:rootlts},
all that remains to do is to group the weights $\tpr_j$ into relevant sets
and upper bound the terms $\max_{i < t \leq T} \Crlts(n_t)$.
\treetopathref{} will again prove useful to relate to the costs of the nodes on the path from $n_1$ to $n_T$.

\begin{proof}[\Proofof{} \Cref{thm:rootlts}]
\newcommand{\cdec}{c^{\mathrm{dec}}}
The following holds for any chosen subtask decomposition $n_1 = n_{T_1}\prec \dots\ n_{T_m} = n_T$.
Define for all $t\leq T$,
with $i<m$ such that $T_i < t \leq T_{i+1}$,
\begin{equation*}
    \cdec(n_t) = \frac{\Crlts_{T_i}(n_t)}{\tpr_{T_i}}\,.
\end{equation*}
The cost function $\cdec$ is an upper bound of $\Crlts$ tailored to the subtask decomposition.
Indeed, for all $n_1 \prec n\preceq n_T$, using the definition of $\Crlts$ in \cref{eq:Crlts},
\begin{align}\label{eq:crlts_leq_cdec}
    \Crlts(n) \leq \min_{i < m} \frac{\Crlts_{T_i}(n)}{\tpr_{T_i}} \leq \cdec(n)\,.
\end{align}
Moreover, by monotonicity of $\Crlts_{T_j}$, 
for $t,j$ with $T_j < t \leq T_{j+1}$
we have $\cdec(n_t) \leq \cdec(n_{T_{j+1}})$,
and combining with \cref{eq:crlts_leq_cdec}:
\begin{align}\label{eq:crlts_leq_max_cdec}
    \Crlts(n_t)\leq \cdec(n_{T_{j+1}}) \leq \max_{j\leq i < m} \cdec(n_{T_{i+1}})\,.
\end{align}
Hence, from \Cref{lem:Crlts_composition},
\begin{subequations}
\begin{align}
    T-1&\leq \sum_{k<T}\tpr_k \max_{k < t \leq T}\Crlts(n_t) \notag\\
    &= \RED{\sum_{j < m}}\ \sum_{\RED{T_j\leq k< T_{j+1}}}\tpr_k \max_{k < t \leq T}\Crlts(n_t) \\
    &\leq \sum_{j < m}\ \sum_{T_j\leq k< T_{j+1}}\tpr_k \max_{\RED{T_j} < t \leq T}\Crlts(n_t) \\
    &=\sum_{j < m}\ \sum_{T_j\leq k< T_{j+1}}\tpr_k \max_{\RED{n_{T_j} \prec n_t \preceq n_T}}\Crlts(n_t) \\
    &\leq \sum_{j < m}\ \sum_{T_j\leq k< T_{j+1}} \tpr_k \max_{\RED{j \leq i < m}}\RED{\cdec(n_{T_{i+1}})}\\
    &= \sum_{j< m} (\Tprm{<T_{j+1}} - \Tprm{< T_j}) \max_{j \leq i < m}
    \frac1{\tpr_{T_i}}(\sopa{n_{T_{i+1}}}{n_{T_i}} - 1)\\
    &\leq \RED{-1} +\sum_{j< m} (\Tprm{<T_{j+1}} - \Tprm{< T_j}) \max_{j \leq i < m}
    \frac1{\tpr_{T_i}}\sopa{n_{T_{i+1}}}{n_{T_i}} 
\end{align}
\end{subequations}
where we have
\begin{eqitems}
\item by partitioning $k$ into the segments of the $m$ subtask decomposition,
\item by relaxing the condition on $t$,
\item by \treetopathref,
\item by \cref{eq:crlts_leq_max_cdec},
\item by definition of $\Tprm{<\cdot}$, and by definition of $\cdec$.
\item by discarding all $-1$ terms for all $j <m-1$,
and using $(\Tprm{<T_m} - \Tprm{< T_{m-1}})/\tpr_{T_{m-1}}\geq 1$ to extract the remaining $-1$ term.
\end{eqitems}
The result follows by adding 1 on both sides.
\end{proof}

\begin{remark}[Minimum subtask decomposition]
Define $\mathcal{D}(n_T)$ to be the set of all subtask decompositions
 of $n_T$.
\Cref{thm:rootlts} is given in the form ``for all subtasks decompositions'',
which is equivalent to
\begin{align*}
    T \leq \min_{\{T_i\}_{i < m} \in \mathcal{D}(n_T)} 
    \sum_{j< m} (\Tprm{<T_{j+1}} - \Tprm{< T_j}) \max_{j \leq i < m}
    \frac1{\tpr_{T_i}}\sopa{n_{T_{i+1}}}{n_{T_i}}\,.
    &\qedhere
\end{align*}
\end{remark}

\section{Comparison between LTS and \rootlts{}}\label{apdx:compare_lts}

To be able to compare the bounds of LTS and \rootlts{} (the simplified bound of \Cref{cor:WTmax_bound}), we can rewrite the LTS bound as a `subtask decomposition'.
Let $n_T$ be the node found at step $T$ by
either LTS or \rootlts{} using the cost function $\sop(\cdot)$, then for any subtask decomposition $\{n_{T_i}\}_{i< m}$ is bounded by
\begin{align*}
    \text{LTS:}\quad T\ &\leq\ \slend(n_T)\, \BLUE{\prod_{i < m}}\, \frac1{\RED{\pol(n_{T_{i+1}}\mid n_{T_i})}}\,, \\[2em]
    \text{\rootlts{}:}\quad T &\ \leq\  
    \,\BLUE{\max_{1 \leq i < m} \frac{\Tprm{<T}}{\tpr_{T_i}}}\,
    \frac{\slend\costargs{n_{T_{i+1}}}{n_{T_i}}}{\RED{\pol(n_{T_{i+1}}\mid n_{T_i})}}\,.
\end{align*}
(Note that the value of $T$ may be differ depending on the algorithm.)
On first reading, omit the terms $\slend(\cdot)$ which are usually small
compared to the other terms --- and recall that $\sopa{\ndn}{n} \leq d(\ndn) - d(n)$.
Each term $\pol(\cdot|\cdot)$ is a product of `action' probabilities,
so $1/\pol(n_{T_{i+1}}|n_{T_i})$ can easily be exponential with the depth of $n_{T_{i+1}}$ relative to $n_{T_i}$.
Now, LTS features a \emph{product} of these terms,
which can itself be exponential with the number of `subtasks' $m-1$ in the considered decomposition.
By contrast, \rootlts{} features a \emph{maximum}, which is a significant improvement over the product in LTS, at the price of the factor $\Tprm{<T} / \tpr_{T_i}$.
A (very) good rerooter could set $\tpr_{T_i} = 1$ for all $i< m$ on the path toward $n_T$,
and $\tpr_t = 0$ everywhere else. Then $\Tprm{<T} = m-1$, and the \rootlts{} bound improves exponentially over the LTS bound.

Of course, the rerooter may not be very good.
In particular, in \Cref{sec:robustness}, it is explained that the factor $\Tprm{<T}$ may be linear with $T$, leading to vacuous bounds, and a transformation of input rerooting weights is proposed to ensure that the bounds do not become vacuous.

Observe also that setting $\tpr_1=1$ and $\tpr_t = 0$ everywhere else makes the \rootlts{} and LTS bounds coincide.

\section{Additional Rerooting Weight Schemes}\label{apdx:weight_schemes}

\Cref{sec:robustness} shows one way to reparameterize the input rerooting weights to provide additional guarantees.
Alternative schemes can be derived.
First, we need this simple result:
\begin{lemma}\label{lem:mvt}
Let $f$ be a differentiable function on $[a, b]$ (possibly with $a=b$), then 
\begin{align*}
    f(b) - f(a) ~\geq~ (b - a) \min_{x\in[a, b]} f'(x)\,.
    &\qedhere
\end{align*}
\end{lemma}
\begin{proof}
Follows directly from the mean value theorem.
\end{proof}

Now, given arbitrary \emph{input} rerooting weights $\{\ttpr_t\}_t\geq 0$, 
for some function $f:[0,\infty) \to [0,\infty)$,
define the rerooting weights for all $t$:
\begin{equation*}
    \tpr_t = f(\TTprm{\leq t}) - f(\TTprm{\leq t-1})\,.
\end{equation*}
Then 
\begin{equation*}
    \Tprm{< T} = f(\TTprm{< T}) - f(0)\,.
\end{equation*}
Using \Cref{lem:mvt}, and observing that $\TTprm{\leq t} -\TTprm{<t} = \ttpr_t$ we also obtain
\begin{equation*}
    \tpr_t \geq \ttpr_t \min_{x\in[\TTprm{< t},\, \TTprm{\leq t}]} f'(x)\,.
\end{equation*}

\begin{example}
Taking $f(x) = -1/\ln(e+x)$ we obtain:
\begin{align*}
    \Tprm{< T} &\leq 1\,,&
    \text{ and }&&
    \tpr_t &\geq \frac{\ttpr_t}{(e+\TTprm{\leq t})\ln^2(e+\TTprm{\leq t})}
\end{align*}
and thus from \cref{cor:WTmax_bound} we get
that for every subtask decomposition $n_1=n_{T_1}\prec\dots\ n_{T_m} = n_T$,
\begin{equation*}
    T \ \leq\  
    \max_{1 \leq i < m}
    \frac{(e+\TTprm{\leq T_i})}{\ttpr_{T_i}}\ln^2(e+\TTprm{\leq T_i})
    \sopa{n_{T_{i+1}}}{n_{T_i}}\,.
\end{equation*}
where the $\ttpr$ factors are fully independent of $T$.
Compare the factor $\ln^2(e+\TTprm{\leq t})$ with the factor 
$1+\ln(\TTprm{< T}/\ttpr_1)$ of \Cref{sec:robustness}.
\end{example}

Some other interesting choices to consider:
$f(x) = \sqrt{x}$,
$f(x) = \ln(1+\alpha x)$ with $\alpha > 0$,
$f(x) = \ln\ln (1+ x)$,
$f(x) = -1/\sqrt{\ln(e+x)}$, 
etc.

\section{Speeding Up the Calculation of the Cost}\label{apdx:calc_cost}

Recall from \cref{eq:Crlts} that the cost of a node $n$ for \rootlts{} is
\begin{align*}
    \Crlts(n) = \min_{n_t \prec n} \frac{\Crlts_t(n)}{\tpr_t}\,.
\end{align*}
The naive calculation of this cost --- as in \Cref{alg:rootlts} (p. \pageref{alg:rootlts}) --- is linear with the depth of $n$.
This may not be particularly problematic if other elements of the search require heavier computation (such as simulating the environment or querying a neural network).
But, using the following result, which exploits a property of $\sop$, it is still possible to reduce this cost to $O(1)$ computation per step for most cases.

\begin{lemma}\label{lem:discard_cfs}
Let $\{n_a, n_b\} \subseteq\anc(n_t)$ be two ancestors of a node $n_t$.
If 
\begin{align*}
    \frac{\Crlts_a(n_t)}{\tpr_a}\leq \frac{\Crlts_b(n_t)}{\tpr_b}
    \quad\text{and}\quad
    \frac{\pol(n_a)}{\tpr_a} \leq \frac{\pol(n_b)}{\tpr_b}
\end{align*}
then for all $n\in\descn(n_t)$,
\begin{align*}
    \frac{\Crlts_a(n)}{\tpr_a}\leq \frac{\Crlts_b(n)}{\tpr_b}\,.
    &\qedhere
\end{align*}
\end{lemma}
This implies that $\Crlts_b(n)/\tpr_b$ never plays a role in the minimum calculation of $\Crlts(n)$.
\begin{proof}
From \Cref{lem:sop_telescopes}
(where we substitute $n_a\leadsto n_a$, $n_b\leadsto n_t$)
it follows that, for all $n\in\descn(n_t)$,
\begin{align*}
    \frac{\Crlts_a(\BLUE{n})}{\pol(n_a)} = \frac{\Crlts_a(\RED{n_t})}{\pol(n_a)}
    + \frac{\Crlts_{\RED{t}}(\BLUE{n})}{\pol(\RED{n_t})}\,,
\end{align*}
and similarly for $b$.
Multiplying all sides by $\pol(n_a) / \tpr_a$ and using the assumptions we obtain
\begin{align*}
    \frac{\Crlts_a(n)}{\tpr_a} 
    \ =\ \frac{\Crlts_a(n_t)}{\tpr_a}
    + \frac{\pol(n_a)}{\tpr_a}\frac{\Crlts_t(n)}{\pol(n_t)}
    \ \leq\ \frac{\Crlts_b(n_t)}{\tpr_b}
    + \frac{\pol(n_b)}{\tpr_b}\frac{\Crlts_t(n)}{\pol(n_t)}
    \ =\ \frac{\Crlts_b(n)}{\tpr_b} \,.
    &\qedhere
\end{align*}
\end{proof}

\Cref{alg:update_cfs} makes use of \Cref{lem:discard_cfs} 
to filter out ancestors that are not needed to calculate the cost of any descendant.

\begin{remark}[Binary weights]\label{rmk:bin_weights_optim}
If $\tpr_t\in\{0,1\}$ for all $t$,
then both conditions of \cref{lem:discard_cfs}
are satisfied if $\tpr_b=1$, for all $n_a \prec n_b$.
This means that the calculation of $\Crlts(n)$ only ever requires considering the deepest ancestor $n_b$ of $n$ with $\tpr_b = 1$.
That is, in \cref{alg:update_cfs}, $|A(n_k)| = 1$ for all $k$.
\end{remark}

\comment{The worst case is triggered in a chain by setting at depth $t$
$\tpr_t = 1/(D+t)$ which (apparently) forces to keep up to $D+2$ cost functions.}

\begin{algorithm}[thbp]
\caption{A simple implementation of the optimization of \cref{apdx:calc_cost}.
The algorithm maintains a subset $A(n_k)$ of the ancestors of $n_k$ needed to calculate the costs of one of its children $n_k$.
Calculating the cost (at step $k$) of $n_t$ and $A(n_t)$ takes time $O(|A(n_k)|)$.
See also \Cref{rmk:num_stability} for additional numerical stability.
\warning{This simple version may miss some removals}
}
\label{alg:update_cfs}
\begin{lstlisting}
# $\lcc{n_k}$: parent of $\lcc{n_t}$
# $\lcc{n_t}$: node for which to calculate the cost
# $\lcc{A(n_k)}$: Subset of the ancestors of $\lcc{n_k}$ needed to calculate
#   the cost of the children of $\lcc{n_k}$
# $\lcc{\pol(n_t|n_k)}$: Policy probability of the transition from $\lcc{n_k}$ to $\lcc{n_t}$
# Returns: the cost of $\lcc{n_t}$ and its subset of ancestors $\lcc{A(n_t)}$.
def cost_and_A($n_k$,  $n_t$, $A(n_k)$, $\pol(n_t|n_k)$):
  # Incremental update of the costs
  $\Big.\Crlts_k(n_t) = 1/\pol(n_t|n_k)$
  for $n_a$ in $A(n_k)$:
    $\Big.\pol(n_t|n_a) = \pol(n_k|n_a)\pol(n_t|n_k)$ 
    $\Big.\Crlts_a(n_t) = \Crlts_a(n_k) + 1/\pol(n_t|n_a)$
  
  $A_t$ = $A(n_k)\cup \{n_k\}$  # temporary set
  
  $\displaystyle a^* = \argmin_{a\in A_t} \frac{\Crlts_a(n_t)}{\tpr_a}$ ; $\displaystyle\Crlts(n_t) = \frac{\Crlts_{a^*}(n_t)}{\tpr_{a^*}}$
  
  # Keep ancestors that cannot be discarded by (*\lcc{\cref{lem:discard_cfs}}*) applied to $\lcc{a\leadsto a^*}$.
  $\displaystyle A(n_t) = \left\{n_b \in A_t :\ \pol(n_{a^*}) / \tpr_{a^*} > \pol(n_b) / \tpr_b \right\}$
  
  return $\Crlts(n_t)$, $A(n_t)$
\end{lstlisting}
\end{algorithm}

\begin{remark}[Numerical stability]\label{rmk:num_stability}
\todo{Add $\slend'$ to table of notation? Only used in this remark though.}
Because the quantities $\pol(\cdot)$ can become exponentially small with the depth,
some care needs to be taken to avoid numerical instability when calculating the costs of the nodes.
Recall that, from \cref{eq:Crlts} and \cref{eq:sop}, for all $n_b\prec n_k \prec n_t$ with $n_t\in\children(n_k)$,
\begin{align*}
    \Crlts_b(n_t) = \Crlts_k(n_t) + \frac1{\pol(n_t|n_b)}\,.
\end{align*}
Define
\begin{align*}
    \slend'\costargs{n_t}{n_b} = \Crlts_b(n_t)\pol(n_t|n_b)\,.
\end{align*}
We recommend maintaining the quantity $\slend'\costargs{n_t}{n_b}$ in linear space, and the quantity $\pol(n_t|n_b)$ in log space.
The quantities can be updated from a node $n_k$ to one of its children $n_t$ as follows:
\begin{align*}
    \slend'\costargs{n_t}{n_b} &= \slend'\costargs{n_k}{n_b}\pol(n_t|n_k) + 1\,, \\
    \ln \pol(n_t|n_b) &= \ln \pol(n_k|n_b) + \ln \pol(n_t|n_k)\,.
\end{align*}
Then the cost can also be calculated in log space --- recall that best-first search is invariant to monotone increasing transformations of the cost function:
\begin{align*}
    \ln \Crlts_b(n_t) &= \ln \slend'\costargs{n_t}{n_b} - \ln \pol(n_t|n_b) \,,\\
    \ln \Crlts(n_t) &= \min_{n_b \prec n_t} \ln \Crlts_b(n_t) - \ln \tpr_b\,.
\end{align*}
Moreover, the condition
\begin{align*}
    \frac{\pol(n_a)}{\tpr_a}\ \leq\ \frac{\pol(n_b)}{\tpr_b}
    \quad\text{is equivalent to}\quad
    \frac{1}{\tpr_a\pol(n|n_a)}\ \leq\ \frac{1}{\tpr_b\pol(n|n_b)}
\end{align*}
for all $n\in\descn(n_a)\cap\descn(n_b)$.
The latter form (in log space) can be more numerically stable as most of the time it requires less precision.
\end{remark}

\section{Detailed Example: Sokoban}\label{apdx:sokoban}

\def\numcells{89}
\def\numstates{207\,538\,210}
\def\actseqlen{25}
\def\sopnT{195\,879\,469}
\def\Mone{1\,143\,408}
\def\Mtwo{125\,496}
\def\Mthree{3\,024}
\def\sopta{733}
\def\soptatb{31}
\def\soptbtc{229}
\def\soptctd{393}
\def\soptb{7213}
\def\soptc{795\,181}
\def\alldecompbound{141\,782\,592}
\def\decompskipta{114\,954\,336}
\def\decompcbound{4\,753\,728}

\begin{figure}[htb]
    \centering
    \includegraphics[width=.4\textwidth]{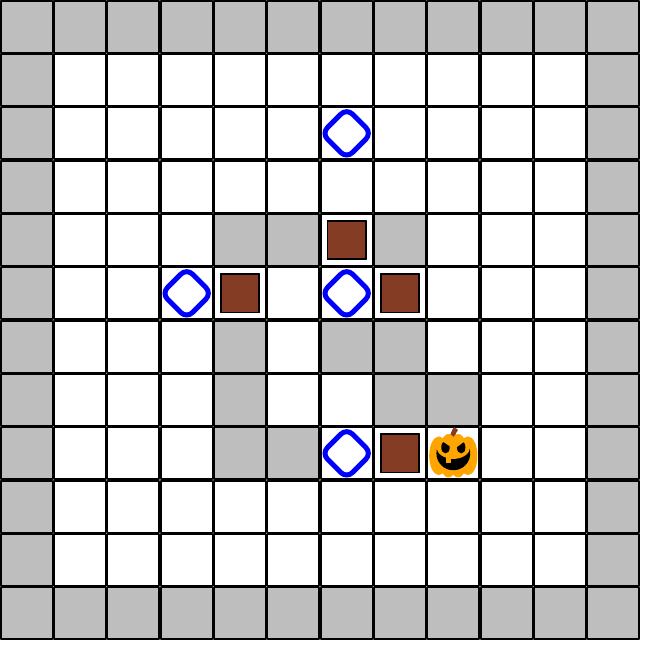}
    \caption{A simple level of Sokoban. The player (the pumpkin) 
    can move in all 4 directions (up, down, left, right) and 
    must push all 4 boxes (brown squares)
    onto one goal spot each (blue diamond cells).
    Boxes cannot be pulled.}
    \label{fig:sokoban}
\end{figure}

To make our ideas a little more concrete, we make some back-of-the-envelope calculations to analyse the behaviour of \rootlts{} on a level of Sokoban --- see \Cref{fig:sokoban}.
The reader must keep in mind that this serves only illustrative purposes.
In particular, this level of Sokoban has been designed to demonstrate some interesting features of our results, and most of the numbers we present are likely largely overestimated. 
Both the rerooter and the policy we design are simplistic for the sake of clarity.

A \emph{state} is one configuration of the board.
Multiple nodes in the search tree may correspond to the same states.
There are \numcells{} non-wall cells, 4 boxes, 4 goal spots and 1 player.
That is, there are \numcells{} possible places for the 5 moving objects and,
since the 4 boxes are not distinguishable,
the number of possible states is at most (but close to) $\numcells{}!/(1!4!(\numcells{}-4-1)!)=\numstates$.
With this kind of Sokoban level configuration, breadth-first search \emph{with transposition tables} --- to ensure visiting each state only once --- likely takes about half this number of node visits on average.

For the level displayed in \Cref{fig:sokoban},
we consider the following solution path.
Starting at the root, the player takes the following actions:\\
\texttt{down left left up up left up up left,\\
  right right up up,\\
  right right down down left,\\
  right down right down down left left,}\\
and ends up in a solution state at a node $n_T = n^*$, where all 4 boxes are on goal spots.
In the sequence above, commas are placed each time a box has been pushed on a goal spot.
The number of actions (= depth of the solution) in this sequence is \actseqlen{}.

We define the policy to be uniform among the actions that lead to states
that differ from the current state and from the previous state.
In particular, this avoids undoing the previous action, such as moving down after moving up --- unless a box has been pushed.
Note that this is significantly weaker than if we had used full transposition tables,
as it does not prevent states from being visited multiple times,
but this allows us to read the policy's probabilities directly from the picture.
Let us rewrite the sequence of actions above, 
with each action preceded by the number of child nodes of positive policy probability:\\
\texttt{
(3)down (3)left (2)left (3)up (2)up (1)left (1)up (1)up (2)left, \\ 
(2)right (2)right (2)up (2)up, \\
(4)right (2)right (3)down (2)down (3)left, \\
(2)right (3)down (1)right (3)down (2)down (3)left (2)left.\\
}
Multiplying all these numbers together and taking the inverse gives the path probability $\pol(n^*)$.
We can calculate $\sop(n^*) = \sop(n_T) = \sopnT$.
This is an upper bound on the number of steps that LTS with this policy needs to perform to find this particular solution node.

Now let us design a simple rerooter.
For the sake of the argument, we assume that no state is visited twice ---
which, as we mentioned earlier, could be enforced at the policy level.
We consider 3 `types' of clues, and each type will have a different associated rerooting weight.
A clue node is of clue type $z\in\{1,2, 3\}$
if exactly $z$ boxes are on a goal spot,
and a box has just been pushed on a goal spot.

Let us make some quick estimates.
There are $N=\numcells$ non-wall cells.
For clues of type 3,
if exactly 3 boxes (chosen out of 4) are on goal spots,
then the remaining box can be in any of the $N-5$ cells (removing the 4 goal positions and the player's position).
For each such positioning of the boxes, the player can be in at most one of 3 cells surrounding each box on a goal spot, so a total of 9 positions.
So our estimate of the number of clues of type 3 is $\binom{4}{3}\binom{N-5}{1}\times 9$.

More generally, for clues of type $z\in\{1,2,3\}$, we estimate the number of clues to be
\begin{equation*}
    M_z = \binom{4}{z}\binom{N-5}{4-z}\times z\times 3\,.
\end{equation*}
where $z\times 3$ is the number of possible positions of the player around one of the $z$ boxes on their goal spots.
We obtain
\begin{align*}
    M_1 &= \Mone\,,&
    M_2 &= \Mtwo\,,&
    M_3 &= \Mthree\,.
\end{align*}
If $n_t$ is a clue node of type $z$ then we set $\tpr_t = 1/M_z$,
otherwise $\tpr_t = 0$.
We also set $\tpr_1 = 1$.
Note that
\begin{equation*}
    \Tprm{<T} \leq 
    \tpr_1 + \sum_{z\leq 3} \text{number of clues of type } z\times 1/M_z
    \leq 4\,.
\end{equation*}

Now let us estimate a bound on the step $T$ at which $n^* = n_T$ is visited for the example of \Cref{fig:sokoban} and the sequence of actions above.

Let $t_1, t_2, t_3$ be the steps at which \rootlts{} visits 
the 3 clue nodes (of types $1,2,3$ in order) of the solution trajectory.

\Cref{thm:rootlts} and \Cref{cor:WTmax_bound} allow us to choose any convenient decomposition into subtasks to calculate a bound,
for a given sequence of actions and the corresponding resulting node $n_T$, \emph{for the analysis, without changing the algorithm}.
So first, let us choose to not decompose at all,
with the subtask decomposition $\{n_1, n_T\}$.
Then, as in \cref{eq:Tw1_bound} we have
\begin{equation*}
    T \leq \Tprm{<T}\sop(n_T) \leq 4 \times \sopnT\,.
\end{equation*}
Obviously, this is not a great bound, as it is a factor 4 worse than the LTS bound.

Second, let us choose a decomposition $\{n_1, n_{t_1}, n_{t_2}, n_{t_3}, n_T\}$
on all 3 clues of each type visited on the solution path.
Then, from \Cref{cor:WTmax_bound},
\begin{align*}
    T &\leq \Tprm{<T} \max\left\{ 
    \sop(n_{t_1}),
    \ \mask{\Mone\times\soptatb}{M_1\sopa{n_{t_2}}{n_{t_1}}},
    \ M_2\sopa{n_{t_3}}{n_{t_2}},
    \ M_3\sopa{n_T}{n_{t_3}}
    \right\}\\
    &=
    \Tprm{<T}
    \max\left\{
    \mask{\sop(n_{t_1})}{\sopta},
    \ \Mone\times\soptatb,
    \ \mask{M_2\sopa{n_{t_3}}{n_{t_2}},}{\Mtwo\times\soptbtc},
    \ \Mthree\times\soptctd 
    \right\} \\
    &\leq \alldecompbound\,.
\end{align*}
While this bound is an improvement over the previous one, as it matches roughly the LTS bound, it still does not offer any advantage.
The culprit is the large number $M_1$:
There are too many clues of type 1, that is, these clues are not informative enough.
So let us choose a decomposition $\{n_1, n_{t_2}, n_{t_3}, n_T\}$
that `skips over' the clue node of type 1, which gives:
\begin{align*}
    T &\leq \Tprm{<T} \max\left\{ 
    \sop(n_{t_2}),
    \ \mask{{}\Mtwo\times\soptbtc{}}{{}M_2\sopa{n_{t_3}}{n_{t_2}}{}},
    \ M_3\sopa{n_T}{n_{t_3}}
    \right\}\\
    &\leq
    \mask{{}\Tprm{<T}{}}{4} \max\left\{ 
    \mask{{}\sop(n_{t_2}){}}{{}\soptb{}},
    \ \Mtwo\times\soptbtc,
    \ \Mthree\times\soptctd
    \right\} \\
    &= \decompskipta\,.
\end{align*}
This bound improves a little bit on the previous one, but not significantly.
$M_2$ is also too large a number.
While there are fewer clue nodes of type 2 than of type 1, 
the information that clue nodes of type 2 provide is not enough
to compensate for the difficulty to find the next clue node (of type 3) from the clue node of type 2.
Let us try one more decomposition $\{n_1, n_{t_3}, n_T\}$ that skips over the first two clue nodes:
\begin{align*}
    T &\leq \Tprm{<T} \max\left\{ 
    \mask{\soptc}{\sop(n_{t_3})},
    M_3\sopa{n_T}{n_{t_3}}
    \right\}\\
    &\leq
    \mask{{}\Tprm{<T}{}}{4} \max\left\{ 
    \soptc,
    \ \Mthree\times\soptctd
    \right\}\\
    &= \decompcbound\,.
\end{align*}
This time we obtain a bound that improves over the LTS bound by a factor 40
--- and, since the clues of types 1 and 2 are always too numerous, it is natural to remove them altogether to save another factor 2 in $\Tprm{<T}$.

Alternatively, for a clue of type $k$ at node $n_t$ we could assign
a rerooting weight $\tpr_t = 1/[(1+q_{k,t})\ln (1+M_k)]$ where $q_{k,t}$ is the number of clues of type $k$ seen up to step $t$ (included).
Then we would still have $\Tprm{<T} \leq 4$, but early clues of type $k$ would have a significantly higher weight than $1/M_k$.

If $M_k$ is unknown or hard to estimate accurately, we can also simply set $\tpr_t = 1/(1+q_{k,t})$, leading to $\Tprm{<T} \leq 1+\sum_k \ln (1+ q_{k,T}) \leq 1+3\ln (1 + \sum_k q_{k,T} / 3)$ (using Jensen's inequality) and thus in the worst case $\Tprm{<T} \leq 1+3\ln (T/3)$.
And, indeed, a numerical simulation (\emph{without} transposition tables) shows that, with this rerooter, the Sokoban level is solved at step $T=1\,853$,
with $\Tprm{<T} < 3.83$ and the numbers of visited clue nodes 
is 8 for type 1, 1 for type 2, and 1 for type 3.
Using this information, we can calculate the bound on the number of visits by decomposing into 4 subtasks:
\begin{align*}
    T \leq 3.83 \max\{\sopta,\ 9\times\soptatb,\ 2\times\soptbtc,\ 2\times\soptctd\} \leq 3\,011\,.
\end{align*}
Observe that this bound is not even a factor 2 away from the actual number of node visits, and is less than a factor 6 of the cube root of the LTS bound $\sop(n_T)= \sopnT$.

\ifdraft
\input{old-and-misc/more-ideas}
\fi

\clearpage
\section{Table of Notation}\label{apdx:table_notation}
\begin{tabular}{ll}
$\Naturals$         & $\{1, 2, \dots\}$ \\
$[N]$               & $\{1,2,\dots N\}$ \\
$\nodeset$          & Set of all nodes, may contain several root nodes \\
$n,\dot n$          & Arbitrary nodes in $\nodeset$\\
$n^*$               & A solution node \\
$\nup,\ndn$         & $n$ ``up'' the tree (ancestor of $n$), $n$ ``down'' the tree (descendant of $n$) \\
$n_1$               & The first node visited by best-first search, global root node  \\
$n_t$               & $t$th node visited by BFS (with some underlying cost function) \\
$d(n)$              & Depth of the node $n$; the root $n_1$ has depth 0 \\
$\slend(n)$         & Slenderness, depends on $\pol$, see \cref{eq:slend} \\
$\children(n)$      & Children of $n$ \\
$\parent(n)$        & Single parent of $n$, except at the root $n_1$\\
$\anc(n)$           & Set of ancestors of $n$ \\
$\ancn(n)$          & $\anc(n)\cup\{n\}$ \\
$\desc(n)$          & Descendants of $n$ \\
$\descn(n)$         & $\desc(n)\cup\{n\}$ \\
$n \prec n', n \preceq n'$ & $n\in \anc(n'), n\in\ancn(n')$ \\
$c$                 & An arbitrary cost function\\
$\Cnm$              & Non-monotone rerooting cost function, see \cref{eq:Cnm} \\
$\Cmax$             & Monotone rerooting cost function, see \Cref{sec:rootlts}\\
$\Crlts$            & Non-monotone rerooting cost function used by \rootlts, see \cref{eq:Crlts}\\
$\nodesettheta$     & Set of nodes  of cost $\sop(\cdot)$ at most $\theta$\\
$\pol(n)$           & $=\pol(n|n_1)$, probability of the node $n$ according to the policy $\pol$, see \Cref{sec:notation} \\
$\pol(n'|n)$        & $\pol(n') / \pol(n)$, assuming $n \preceq n'$ \\
$\dop(n)$           & $d(n)/\pol(n)$, cost function used in original LTS \\
$\sop(n)$           & $\slend(n)/\pol(n) \leq 1+\dop(n)$, self-counting cost function, see \cref{eq:sop} \\ 
$\sopa{n}{n^a}$     & Rooted version of $\sop$, see \cref{eq:sop_anc_root}\todo{Use $\sop(n|n^a)$ for clarity?}\\
$\tpr_t$            & Rerooting weight assigned to the node $n_t$ upon visiting it at step $t$\\
$\Tprm{\leq t}$     & Cumulative rerooting weight $\Tprm{\leq t} = \sum_{k\leq t} \tpr_k$\\
$\ttpr$             & Reparameterization of $\tpr$, input rerooting weights\\
$q$                 & Fixed number of clues \\
$q_t$               & Number of clue nodes visited up to (including) step $t$\\ 
$\clueset$          & Set of clue nodes \\
$\indicator{test}$  &=1 if $test$ is true, 0 otherwise
\end{tabular}

\ifdraft
\input{old-and-misc/log-loss-rerooter-policy}
\fi

\end{appendix}

\end{document}